\documentclass[10pt]{article}

\usepackage{arxiv_styles}

\title{Batched Kernelized Bandits: Refinements and Extensions}

\author{
Chenkai Ma \\
National University of Singapore \\
\texttt{chenkai.ma@u.nus.edu}
\and
Keqin Chen \\
National University of Singapore \\
\texttt{keqin.chen@u.nus.edu}
\and
Jonathan Scarlett \\
National University of Singapore \\
\texttt{scarlett@comp.nus.edu.sg}
}

\date{}

\begin{document}

\maketitle

\begin{abstract}
    In this paper, we consider the problem of black-box optimization with noisy feedback revealed in batches, where the unknown function to optimize has a bounded norm in some Reproducing Kernel Hilbert Space (RKHS).
    We refer to this as the Batched Kernelized Bandits problem, and refine and extend existing results on regret bounds. 
    For algorithmic upper bounds, \citep{Li22} shows that $B=O(\log\log T)$ batches suffice to attain near-optimal regret, where $T$ is the time horizon and $B$ is the number of batches. We further refine this by (i) finding the optimal number of batches \emph{including constant factors} (to within $1+o(1)$), and (ii) removing a factor of $B$ in the regret bound.  
    For algorithm-independent lower bounds, noticing that existing results only apply when the batch sizes are fixed in advance, we present novel lower bounds when the batch sizes are chosen adaptively, and show that adaptive batches have essentially same minimax regret scaling as fixed batches.
    Furthermore, we consider a robust setting where the goal is to choose points for which the function value remains high even after an adversarial perturbation.  We present the robust-BPE algorithm, and show that a suitably-defined cumulative regret notion incurs the same bound as the non-robust setting, and derive a simple regret bound significantly below that of previous work.
\end{abstract}

\section{Introduction}
Optimizing a black-box function with costly and noisy evaluations is a common problem, and has wide applications such as hyperparameter tuning \citep{Snoek12}, A/B testing \citep{Kohavi09}, recommendation systems \citep{Li10}, and reinforcement learning \citep{Lizotte07}. 
Bayesian optimization (BO) has been a highly successful approach that assumes a prior distribution on the function, sequentially observes noisy data to compute the posterior distribution of the function, and uses the posterior for decision-making \citep{Garnett23}.
BO has been widely considered in two distinct settings: the Bayesian setting that assumes the function is drawn from a Gaussian process (GP), and the non-Bayesian setting that assumes the function lies in a Reproducing Kernel Hilbert Space (RKHS) and has bounded RKHS norm. In this paper we focus on the latter, which is also known as kernelized bandits. More precisely, we consider a practically motivated variation where the noisy observations are not revealed immediately, but in batches, which we refer to as the \emph{batched kernelized bandits}.  The motivation is that collecting samples in batches can be beneficial or even essential since they can be done parallel \citep{Kittur08, Bertsimas07}.

In the case that the batch sizes are fixed (i.e., must all be chosen prior to taking any samples), \citep{Li22} presents the BPE algorithm, and shows that it achieves near-optimal $O^*(\sqrt{T\gamma_T})$ regret with $B=O(\log\log T)$ batches, where $T$ is the time horizon, $\gamma_T$ is the maximum information gain (see Section \ref{sec:setup_standard}), and $O^*$ hides dimension-independent logarithmic factors. The upper bounds are also accompanied by near-matching lower bounds for common kernels.

In this paper, we continue this line of work by refining and extending the results on batched kernelized bandits. 
Our contributions are as follows:
\begin{itemize}
    \item For algorithmic upper bounds, we refine the upper bounds in \citep{Li22} by modifying the BPE algorithm therein. 
    In particular, we identify the optimal number of batches not only in the $O(\cdot)$ sense but including precise constants (to within $1+o(1)$), and we remove an unnecessary factor of $B$ in the regret bound.  Some simple supporting experiments are also given in Appendix \ref{app:experiments}.
    \item For algorithm-independent lower bounds, we present novel lower bounds that hold even when batch sizes are chosen adaptively. 
    The bounds have an additional polynomial dependence on $B$, but this is mild since we can assume $B=O(\log\log T)$ and rely on standard (fully sequential) lower bounds for larger $B$.  Our findings show that despite being more general, adaptive batches are essentially no more powerful than fixed batches in the minimax sense.
    \item We extend batched kernelized bandits to a robust setting, where the goal is to choose points for which the function value remains high even after an adversarial perturbation. In this setting, we present the robust-BPE algorithm, and show that it achieves a (suitably defined) cumulative regret matching that of the non-robust setting, and also significantly improves over \citep{Bogunovic18} in terms of simple regret.
\end{itemize}

\subsection{Related Work} \label{sec:related}

\textbf{Upper bounds for kernelized bandits.} The celebrated work of \cite{Srinivas10} presents GP-UCB, an algorithm that chooses the next point to evaluate based on upper confidence bounds \citep{Auer02, Auer02b, Dani08}. It is shown that GP-UCB achieves $O^*(\sqrt{T}\gamma_T)$ regret with high probability, where $O^*$ suppresses dimension-independent logarithmic factors, and $\gamma_T$ is the maximum information gain. A problem with the $O^*(\sqrt{T}\gamma_T)$ scaling is that it can fail to be sublinear for the Mat\'ern kernel for many configurations of the smoothness parameter $\nu$ and dimension $d$.

Attempts to address this issue come mainly in two directions. In the first direction, several ensuing works improve the regret bounds to $O^*(\sqrt{T\gamma_T})$ \citep{Valko13, Janz20, Salgia21, Camilleri21, Li22, Vakili22, Salgia24}. For instance, \citep{Valko13} presents the SupKernelUCB algorithm, which is a kernelized version of an LinUCB \citep{Chu11}, and combines upper confidence bounds with action elimination; \citep{Li22} considers a different approach (which is also the simplest approach to attain near-optimal regret) by iterating between sampling points with maximum uncertainty and eliminating suboptimal points using confidence bounds, which crucially relies on a sharp confidence interval from \citep{Vakili21a}.\footnote{A similar sharp confidence interval can also be inferred from \citep{Valko13}.} In the second direction, several works improved the bounds on $\gamma_T$ \citep{Janz20, Vakili21b}. For example, for the Mat\'ern kernel, \citep{Srinivas10} has $\gamma_T = O \bigl( T^{\frac{d(d+1)}{2\nu + d(d+1)}} (\log T) \bigr)$, which \citep{Vakili21b} improves to $\gamma_T = O \bigl( T^{\frac{d}{2\nu+d}} (\log T)^{\frac{2\nu}{2\nu+d}} \bigr)$.

Combining the $O^*(\sqrt{T\gamma_T})$ regret bound with improved regret bounds on $\gamma_T$ leads to $O^*\bigl(\sqrt{T(\log T)^d}\bigr)$ for the SE kernel, and $O^*\bigl(T^{\frac{\nu+d}{2\nu+d}}(\log T)^{\frac{\nu}{2\nu+d}}\bigr)$ for the Mat\'ern kernel, which is sublinear for general $(\nu,d)$. 
As we will discuss, these upper bounds nearly match the corresponding lower bounds, which means $O^*(\sqrt{T\gamma_T})$ regret is near-optimal.
More recently, \citep{Iwazaki2025} improves the regret upper bounds in terms of their dependence on the RKHS norm, which is orthogonal to our work since we consider the RKHS norm to be a fixed constant.

\textbf{Lower bounds for kernelized bandits.} As an early work, \citep{Bull11} presents tight lower bounds on the simple regret for the Mat\'ern kernel in the noiseless setting. \citep{Scarlett17} extends the results in \citep{Bull11} to the noisy setting, and to the SE kernel, and considered both simple regret and cumulative regret. Specifically, \citep{Scarlett17} shows that the expected cumulative regret is lower bounded by $\Omega\bigl(\sqrt{T(\log T)^{d/2}}\bigr)$ for the SE kernel, and $\Omega\bigl(T^{\frac{\nu+d}{2\nu+d}}\bigr)$ for the Mat\'ern kernel. \citep{Cai21} complements \citep{Scarlett17} by presenting high probability regret lower bounds of the same order as well as extensions to two robust settings. More recently, \citep{Li22} presents lower bounds for batched kernelized bandits in the regime $B = O(1)$, while relying on the general (not necessarily batched) lower bounds in \citep{Scarlett17} for growing $B$.

\textbf{Batched settings.} 
Some earlier works have studied batched kernelized bandits where every batch has some fixed length \citep{Contal13, Desautels14, Kandasamy18}, which means the number of batches is $\Theta(T)$. In contrast, we are interested in making the number of batches small (which is often highly desirable in practice, e.g., in clinical trials \citep{Gao19}), and we allow the batch sizes to differ (e.g., increase over time).
Then, the most relevant existing work is \citep{Li22}, which characterizes both upper bounds and lower bounds for cumulative regret when the batches lengths are pre-specified (i.e., fixed). As discussed above, the upper bounds are near-optimal when $B = O(\log \log T)$ (due to matching \citep{Scarlett17}), and refined lower bounds are also derived for $B = O(1)$.

However, a notable gap in \citep{Li22} is that no lower bounds are given for \emph{adaptively-chosen} batch sizes, which in principle could lead to algorithms with lower regret.  The difficulty of lower bounds for adaptive batches is addressed in the finite-arm multi-armed bandit problem in \citep{Gao19}, which presents minimax and problem-dependent lower bounds for both fixed batches and adaptive batches. The main idea of \citep{Gao19} is to ``change measure'', i.e., reduce the adaptive batches chosen by an algorithm to ``reference batches'' that we can control and have nice properties, and relate the probabilities on different reward distributions to a simpler ``base'' distribution. We similarly ``change measure'' in our analysis, but in an arguably simpler way, as we consider $\Theta(BM)$ hard functions where $M$ is the number of disjoint regions of the input domain by some partition, whereas they consider $\Theta(BM^2)$ hard instances in the case of $M$ arms.
Moreover, our analysis has substantially different details due to the infinite-arm nature of our setup (as opposed to finitely many independent arms).
Earlier works on batched multi-armed bandits include \citep{Cesa-Bianchi13, Perchet16, Esfandiari21}, and our choice of batch sizes is inspired by \citep{Cesa-Bianchi13}.

\textbf{Robust settings.} The robust setting we consider extends kernelized bandits by seeking points whose function values remain high even after an adversarial perturbation. This setup is introduced in \citep{Bogunovic18}, and while other robust settings also exist (e.g., \citep{Bogunovic20}), they are beyond our scope. 
In this setting and other related settings, it is most common to focus on the simple regret $r_{\xi}(x^{(T)})$, which is incurred on a single point $x^{(T)}$ reported after $T$ rounds, where $\xi$ is a constant denoting the extent of perturbation \citep{Bogunovic18, Sessa20, Nguyen20}.
In this paper, however, we consider a suitably defined notion of cumulative regret $R_{\xi, T}$, which ensures consistency with our regret bounds for (non-robust) batched kernelized bandits. Moreover, noting that the best simple regret is no higher than the best normalized cumulative regret \citep{Bubeck09}, our result directly leads to a near-optimal bound on simple regret as well.

\section{Problem Setup}
\label{sec:setup}
We first introduce the setup for batched kernelized bandits, which we refer to as the standard setting. Then we introduce the robust setting discussed above.

\subsection{Standard Setting}
\label{sec:setup_standard}
We consider the problem of sequentially optimizing a black-box function $f(x)$ over some domain $\mathcal{X} \subseteq [0,1]^d$ with time horizon $T$ and number of batches $B$. To avoid trivial cases we require $T \geq B \geq 2$.

Given $T$ and $B$, an algorithm splits $T$ into $B$ batches by choosing times $t_1, \ldots, t_{B}$, subject to $0 < t_1 < \cdots < t_B = T$, and for convenience we define $t_0 = 0$. This means the $i$-th batch is $\{t:t_{i-1}< t\leq t_i \}$, and has batch size $N_i = t_i -t_{i-1}>0 $.\footnote{\label{fn:handle_extra_batches} 
The algorithm designer may want to choose batch sizes as a function of $T$ without explicitly being given $B$ (e.g., see \eqref{eq:grow_b_batch_sizes_li} and \eqref{eq:grow_b_batch_sizes} below).  In such scenarios, we assume for convenience that $B$ equals the exact number of batches needed for such a choice.  (An alternative approach would be to say that $B$ is allowed to be higher and some batch lengths are zero, but this would complicate the notation.)}
We denote $[T]=\{1,\ldots,T\}$. At the \emph{start} of each time $t \in [T]$, an algorithm selects a point $x_t\in\mathcal{X}$ to query $f$, which incurs noisy response $y_t=f(x_t)+z_t$, where the noise $z_t\sim N(0, \sigma^2)$ is independent for all $t\in[T]$.\footnote{For our upper bounds, we can trivially generalize $z_t$ to be i.i.d~$\sigma$-sub-Gaussian for all $t\in[T]$.} 
We note that $y_t$ may not be immediately observable to the algorithm at the \emph{end} of time $t$, i.e., after choosing $x_{t}$. Instead, if $t \in (t_{i-1}, t_i]$, then the algorithm observes $y_t$ at the end of time $t_i$.
Depending on when the $t_i$ values are chosen, we consider two settings:
\begin{itemize}
    \item Fixed batches: all $t_i$ are predetermined.
    \item Adaptive batches: $t_1$ is predetermined, and for $i\geq 2$, $t_i$ is chosen at the end of time $t_{i-1}$ (i.e., after receiving the observations from the $(i-1)$-th batch).
\end{itemize}

The algorithm aims to find $x^*\in\argmax_{x\in\mathcal{X}} f(x)$, an (arbitrary) maximizer of $f$, which we measure by cumulative regret $R_T = \sum_{t=1}^T r_t $, where $r_t = f(x^*) - f(x_t)$. We also define $R^i =\sum_{t=t_{i-1}+1}^{t_i} r_t $ to be the \emph{cumulative regret} in the $i$-th batch.  In some cases we also consider $r^{(T)} = f(x^*) - f(x^{(T)})$, the \emph{simple regret} of an additional point $x^{(T)}$ output after $T$ rounds. 

As for the smoothness of $f$, we assume $\lVert f\rVert_k \leq \Psi$, which means $f$ lives in an RKHS where the kernel is $k$, and its RKHS norm is upper bounded by $\Psi > 0$. We also assume that the kernel is normalized such that $k(x,x)\leq 1, \forall x\in\mathcal{X}$.
We consider two types of kernels, namely the squared exponential (SE) kernel and the Mat\'ern kernel:
\begin{align*}
    k_{\text{SE}}(x,x') & = \exp\biggl(-\frac{(d_{x,x'})^2}{2l^2}\biggr) \\
    k_{\text{Mat\'ern}}(x,x') & = \frac{2^{1-\nu}}{\Gamma(\nu)} \biggl(\frac{\sqrt{2\nu}d_{x,x'}}{l}\biggr)^{\nu} J_{\nu}\biggl(\frac{\sqrt{2\nu}d_{x,x'}}{l}\biggr),
\end{align*}
where $d_{x,x'}=\lVert x-x'\rVert$, $l>0$ is the length-scale, $\nu > 0$ is a smoothness parameter, $\Gamma$ is the Gamma function, and $J_{\nu}$ is the modified Bessel function.
We let $\gamma_t = \max_{A\subseteq \mathcal{X}:\lvert A\rvert=t} I(f_A;y_A) $ be the maximum information gain \citep{Srinivas10}, where $f_A=[f(x_t)]_{x_t\in A} $, $y_A=[y_t]_{x_t\in A} $, and $I(\cdot;\cdot)$ is the mutual information (\cite{Cover06}, Sec. 2.3). Two useful known upper bounds for $\gamma_t$ are
\begin{align}
    & \gamma_t^{\text{SE}} = O\bigl((\log t)^{d+1}\bigr), \label{gamma_t_se} \\
    & \gamma_t^{\text{Mat\'ern}} = O \bigl( t^{\frac{d}{2\nu+d}} \left( \log t \right)^{\frac{2\nu}{2\nu +d}} \bigr), \label{gamma_t_mat}
\end{align}
where the first bound is from \citep{Srinivas10}, and the second is from \citep{Vakili21b}.

Throughout the paper, we treat the parameters $(d,\nu,l, \Psi, \sigma)$ as fixed and known constants, and focus on how the regret depends on $T$ and $B$. For the dependence on $B$, we will differentiate between ``constant $B$'', where $B$ remains constant as $T$ grows,
and ``growing $B$'', where $B$ grows with $T$ as a function of $T$, e.g., $B= O(\log\log T)$.
We use $O^*(\cdot)$ to hide dimension-independent logarithmic factors in $T$, and use $O(\cdot),\Omega(\cdot),\Theta(\cdot), o(\cdot)$, $\omega(\cdot)$ as commonly defined.

\subsection{Robust Setting}
\label{sec:setup_robust}
The robust setting inherits the standard setting, but with the modified goal of selecting points whose function value remains high even after a suitably-constrained adversarial perturbation  \citep{Bogunovic18}.
More formally, we consider $d:\mathcal{X}\times\mathcal{X} \to \mathbb{R}$, a known function that measures the ``distance'' between elements in $\mathcal{X}$, and let $\xi$ be a constant.
For each $x\in\mathcal{X}$, we define $\Delta_{\xi}(x) = \{x'-x:x'\in\mathcal{X},d(x,x')\leq \xi\} $, with the understanding that $x+\Delta_\xi(x)$ is the set of possible perturbed queries subject to a ``distance'' upper bound $\xi$.
The goal is to find a $\xi$-robust optimal point:
\begin{align}
    x^* \in \argmax_{x\in{\mathcal{X}}}\min_{\delta\in\Delta_{\xi}(x)}f(x+\delta).
\end{align}
At time $t$, the algorithm chooses $x_t$ and $\delta_t$, queries $f$ at $x_t+\delta_t$, and incurs a noisy response $y_t = f(x_t+\delta_t)+z_t $ that is revealed at time $t_i $ if $t\in (t_{i-1}, t_i]$. 
Then, we measure a tuple $(x_t,\delta_t)$ by the $\xi$-regret:
\begin{align}
    r_{\xi}(x_t) = \min_{\delta\in\Delta_{\xi}(x^*)}f(x^* + \delta) - \min_{\delta\in\Delta_{\xi}(x_t)}f(x_t+\delta), \label{eq:xi_regret}
\end{align}
which is defined with respect to $x_t$ only, and considers the worst-case $\delta$ instead of $\delta_t$. % , as though $x_t$ were a reported point.
Similarly, for the sequence $(x_1,\delta_1), \ldots, (x_T,\delta_T)$ selected by the algorithm in time horizon $T$, the cumulative $\xi$-regret is defined as $R_{\xi,T} = \sum_{t=1}^T r_{\xi}(x_t) $. We also define $R^i_{\xi}=\sum_{t=t_{i-1}+1}^{t_i} r_{\xi}(x_t)$ to be the cumulative $\xi$-regret in the $i$-th batch.  Along with the cumulative regret, we will consider the simple regret $r_{\xi}(x^{(T)})$ of an additional point $x^{(T)}$ reported after $T$ rounds, as was considered in \citep{Bogunovic18}.

The preceding cumulative regret measure may appear somewhat non-standard in that it consists of the algorithm choosing both $x_t$ and $\delta_t$ but only measuring regret with respect to $x_t$.  The role of $\delta_t$ is to ``simulate'' perturbation from the adversary, and a similar notion was used in a multi-agent problem in \citep{Tay23}.  Importantly, even if one is only interested in the simple regret $r_{\xi}(x^{(T)})$, the cumulative regret $R_{\xi,T}$ will still serve as a helpful stepping stone towards that (see Remark \ref{rmk:robust_simple_regret}).

\section{Standard Setting: Upper Bounds}
\label{sec:standard_upper_bound}

In the standard setting, we will first give refined algorithmic upper bounds with fixed (i.e., pre-specified) batch sizes.\footnote{We do not consider adaptive batches here, as using fixed batches already leads to near-optimal regret; see Theorem \ref{thm:upper_grow_b}.}
We focus on results with growing $B$ in this section, and defer results with constant $B$ to Appendix \ref{app:standard_fixed_B_upper_bounds}.
In line with \citep{Li22}, we consider $\mathcal{X}$ to be finite but possibly very large, and note that extending the results to the continuous domain $\mathcal{X}=[0,1]^d$ is standard (see Remark \ref{rmk:extent_to_continuous_domain} below).

\subsection{Preliminaries and Existing Results}
\label{sec:standard_upper_bound_li}
\citep{Li22} presents the Batched Pure Exploration (BPE) algorithm. 
BPE maintains a set of potential maximizers $\mathcal{X}_i$. In each batch, BPE explores points in $\mathcal{X}_i$ with maximum posterior variance. Then, at the end of the batch, BPE computes confidence bounds to eliminate suboptimal points from $\mathcal{X}_i$. The updated $\mathcal{X}_i$ is used to initialize the next batch.
BPE is non-adaptive in each batch, which allows the use of tight confidence bounds from \citep{Vakili21a}, which turn out to be crucial for near-optimal regret.
More details of the BPE algorithm are given in Appendix \ref{app:bpe}.
The performance of BPE crucially depends on the batch sizes (Proposition \ref{prop:property_bpe} in the appendix), and for growing $B$, \citep{Li22} considers 
\begin{align}
    N_i=\min \biggl\{ \big\lceil \sqrt{TN_{i-1}} \big\rceil, T-\sum_{j=1}^{i-1}N_j\biggr\}, \label{eq:grow_b_batch_sizes_li}
\end{align}
with $N_0=1$ and $i\in [B]$. 
For these batch sizes, \citep{Li22} shows that BPE terminates with no more than $\lceil \log_2\log_2 T \rceil + 1$ batches (i.e., $B \leq \lceil \log_2\log_2 T \rceil + 1$), 
and achieves the following.

\begin{lemma}
    \emph{\citep[Theorem 1]{Li22}}
    \label{lem:upper_grow_b_li}
    Under the setup of Section \ref{sec:setup_standard} and using the batch sizes defined in \eqref{eq:grow_b_batch_sizes_li},
    for any $\delta \in (0,1)$,
    the BPE algorithm yields with probability at least $1-\delta$ that
    \begin{align*}
    	R_T=O \Bigl((\log\log T) \bigl(\Lambda + \sqrt{\log \log\log T}\bigr) \sqrt{T\gamma_T}\Bigr),
    \end{align*}
    where $\Lambda=\Psi+\sqrt{\log(|\mathcal{X}|/\delta)}$.
\end{lemma}
\begin{remark}
    \label{rmk:extent_to_continuous_domain}
    \citep{Li22} discusses some approaches \citep{Janz20, Vakili21a} to extend the result to the continuous domain $\mathcal{X}=[0,1]^d$. For example, \citep{Janz20} uses a discretization argument and the fact that $f$ is Lipschitz continuous \citep{Lee22} to get an upper bound of the same $T$ scaling as in that of finite $\mathcal{X}$, up to log factors.
\end{remark}

An implication of Lemma \ref{lem:upper_grow_b_li} is that it is sufficient to have $B = \lceil \log_2\log_2 T \rceil + 1 = O(\log\log T)$ batches to recover the near-optimal $O^*(\sqrt{T\gamma_T})$ regret.

\subsection{Refinements}
We refine the regret bounds in Lemma \ref{lem:upper_grow_b_li} by generalizing the batch sizes in \eqref{eq:grow_b_batch_sizes_li}. Specifically, we consider
\begin{align}
    N_i = \min \biggl\{\lceil T^{1-a^i} \rceil, T-\sum_{j=1}^{i-1}N_j\biggr\}, \label{eq:grow_b_batch_sizes}
\end{align}
where $a\in(0,1)$ is a parameter to the algorithm and $i\in[B]$. By ``generalizing'', we mean that choosing $a = 1/2$ is equivalent to \eqref{eq:grow_b_batch_sizes_li}; also note that $a=0$ recovers the single batch setting, and $a=1$ recovers the fully sequential setting. 
Similar to \eqref{eq:grow_b_batch_sizes_li}, we can show that the number of batches in \eqref{eq:grow_b_batch_sizes} satisfies (see the proof of Proposition \ref{prop:grow_b_batch_sizes_property}):
\begin{align}
    B \leq \lceil \log_{1/a}\log_2 T \rceil + 1, \quad  
    B  \geq \Bigl\lfloor \log_{1/a}\Bigl(\frac{\log T}{\log\log T}\Bigr) \Bigr\rfloor, \label{eq:B_upper}
\end{align}
and hence 
\begin{equation}
    B = \bigl(\log_{1/a}\log T\bigr) (1+o(1)) = \frac{\log_2\log_2 T}{\log_2(1/a)}(1+o(1)). \label{eq:B_asymp}
\end{equation}
Compared to \eqref{eq:grow_b_batch_sizes_li}, our upper bound on $B$ is tighter when $a \in (0,\frac{1}{2})$, and will be complemented by a matching lower bound establishing that the constant in $B = \bigl(\log_{1/a}\log T\bigr) (1+o(1))$ cannot be reduced further (see Remark \ref{rmk:necessary_B} below). 
Furthermore, the following theorem provides a slightly better regret bound for $a\in (\frac{1}{2}, 1)$, along with an analog for $a\in(\frac{\nu}{2\nu+d},\frac{1}{2}]$ when a ``well-behaved'' upper bound on $\gamma_T$ is used.

\begin{theorem}
    \label{thm:upper_grow_b}
    Under the setup of Section \ref{sec:setup_standard} and
    using batch sizes defined in \eqref{eq:grow_b_batch_sizes} with $a\in(\frac{1}{2},1)$ for either the SE kernel or the Mat\'ern kernel,
    for any $\delta \in (0,1)$,
    the BPE algorithm yields with probability at least $1-\delta$ that
    \begin{align*}
    	R_T
        = O \Bigl(\bigl(\Lambda + \sqrt{\log \log\log T}\bigr)\sqrt{T\gamma_T}\Bigr),
    \end{align*}
    where $\Lambda=\Psi+\sqrt{\log(|\mathcal{X}|/\delta)}$.
    In addition, using batch sizes defined in \eqref{eq:grow_b_batch_sizes} with $a\in(\frac{\nu}{2\nu+d},\frac{1}{2}]$ for the Mat\'ern kernel,
    BPE yields with probability at least $1-\delta$ that
    \begin{align*}
    	R_T
        = O \Bigl(\bigl(\Lambda + \sqrt{\log \log\log T}\bigr)\sqrt{T\overline{\gamma}_T}\Bigr),
    \end{align*}
    where $\overline{\gamma}_t$ is any sequence that (i) upper bounds $\gamma_t$ for sufficiently large $t$, and (ii) takes the form $\overline{\gamma}_t = c't^{\frac{d}{2\nu+d}}\alpha_t$ for some non-decreasing $\{\alpha_t\}_{t \ge 1}$ and constant $c' > 0$.  (See also Definition \ref{def:gamma_t} in the appendix for a more general notion of a ``well-behaved'' upper bound in which $\alpha_t$ need not be non-decreasing.)
\end{theorem}

\begin{remark}
    \label{rmk:choose_a_and_more}
    For batch sizes in \eqref{eq:grow_b_batch_sizes}, we have the following observations on the effect of $a$ on $R_T$ and $B$:
    \begin{itemize}
        \item Note that \eqref{eq:grow_b_batch_sizes_li} is equivalent to $a=\frac{1}{2}$.  
        By taking $a$ arbitrarily close to $\frac{1}{2}$, we can attain the same bound on $B$ as that of Lemma~\ref{lem:upper_grow_b_li} (or if $\log_2\log_2 T$ is precisely an integer, only increase it by one) while slightly improving $R_T$.
        \item For the Mat\'ern kernel, choosing $a\in(\frac{\nu}{2\nu+d},\frac{1}{2}]$ (and arbitrarily close to $\frac{\nu}{2\nu+d}$) implies $B \leq \lceil \log_{1/a}\log_2 T \rceil + 1 \leq \lceil \log_2\log_2 T \rceil + 1$, meaning we have a better bound on $B$ (see also Remark~\ref{rmk:necessary_B} below on its near-optimality), i.e., fewer batches. 
        In this case, for technical reasons, the bound is expressed in terms of a ``well-behaved'' upper bound $\overline{\gamma}_T$, instead of $\gamma_T$ itself.  We note that (i) the state-of-the-art bound \eqref{gamma_t_mat} satisfies this condition, (ii) more general conditions are given in the appendix (Definition \ref{def:gamma_t}), (iii) it is very plausible that $\gamma_t$ itself satisfies such conditions, but this may be difficult to prove, and (iv) in any case, the difference is at most a dimension-independent log term (see Remark \ref{rmk:near-optimal} in the appendix).
    \end{itemize}
    \end{remark}
    
    \begin{remark}
    For the Mat\'ern kernel, a smaller $\frac{\nu}{d}$ means a ``harder problem'' for the algorithm.  A smaller $\frac{\nu}{d}$ allows $a$ to be smaller, since $a\in(\frac{\nu}{2\nu+d}, 1)$, which by \eqref{eq:B_asymp} means a lower value of $B$. This means that the algorithm needs fewer batches to achieve the near-optimal regret as the problem gets harder.  Intuitively, the reason for this is that for a harder problem, the algorithm needs larger batches to learn anything, and this amounts to fewer batches overall. 
    This advantage is not present in \eqref{eq:grow_b_batch_sizes_li} and Lemma \ref{lem:upper_grow_b_li}, where the bound on $B$ only depends on $T$ but not $(\nu, d)$.
\end{remark}
\begin{remark}
    Compared to Lemma \ref{lem:upper_grow_b_li}, our regret bound avoids the leading $\log\log T$ term, which appeared in the analysis of \citep{Li22} as multiplication by the number of batches.  
    This term is arguably fairly insignificant, but its removal still has some interesting implications.  
    As a useful point of comparison, we note that in the Bayesian setting, the upper bound on cumulative regret is $R_T = O\bigl((\sqrt{\log (\lvert \mathcal{X}\rvert/\delta) } + \sqrt{\log T})\sqrt{T\gamma_T}\bigr) $ from \citep{Srinivas10}. Then, under the mild assumption $\log T = O(\log\lvert \mathcal{X}\rvert)$,
    the bound simplifies to $O\bigl(\sqrt{T\gamma_T\log(\lvert \mathcal{X}\rvert/\delta)} \bigr)$.  By similar reasoning, our result only requires the \emph{extremely mild} assumption 
    $\log\log\log T = O(\log\lvert \mathcal{X}\rvert)$
    to simplify to $O\bigl(\sqrt{T\gamma_T\log(\lvert \mathcal{X}\rvert/\delta)} \bigr)$. This scaling was not achieved by \cite{Li22} under either assumption.\footnote{Higher logarithmic factors are present in \citep{Valko13} and \citep{Salgia21}, whereas \citep{Camilleri21} attains this scaling when $\log\log T = O(\log\lvert \mathcal{X}\rvert)$.   As discussed in \citep{Li22}, the BPE approach has the desirable property of being arguably the most elementary among these works with near-optimal regret. }
\end{remark}

The proof of Theorem \ref{thm:upper_grow_b} is given in Appendix \ref{app:proof_thm_upper_grow_b}.  
Notably, we use the batch sizes in \eqref{eq:grow_b_batch_sizes} to avoid incurring a common bound for cumulative regret in each batch (which would lead to Lemma \ref{lem:upper_grow_b_li}), and instead show that the regret bound from the final batch dominates those from all previous batches (see \eqref{eq:ub_dominate}).

\section{Standard Setting: Lower Bounds}
\label{sec:standard_lower_bound}

For our algorithm-independent lower bounds, we allow $B$ to be either constant or growing with $T$, and consider $\mathcal{X} = [0,1]^d$.
Our main contribution is a novel lower bound on expected cumulative regret for adaptive batches, which is accompanied by a high probability lower bound of the same order. For context, we will first discuss existing lower bounds for fixed batches in \citep{Li22}.

\subsection{Existing Results for Fixed Batches}
\label{sec:adaptive_batches_preliminary}
\begin{lemma}
    \emph{\citep[Theorem 3]{Li22}}
    \label{lem:lower_fixed_batches}
    Under the setup of Section \ref{sec:setup_standard},
    for any batched algorithm with a constant number $B$ of fixed batches, there exists $f$ with $\lVert f\rVert_k \leq \Psi$ such that with probability at least $\frac{3}{4}$:
    \begin{enumerate}
        \item for the SE kernel, it holds that $R_T=\Omega\Bigl(T^\frac{1-\eta}{1-\eta^B}(\log T)^{\frac{d(\eta-\eta^B)}{2(1-\eta^B)}} \Bigr)$, where $\eta=\frac{1}{2}$;
        \item for the Mat\'ern kernel, it holds that $R_T=\Omega\bigl(T^\frac{1-\eta}{1-\eta^B}\bigr)$, where $\eta=\frac{\nu}{2\nu+d}$.
    \end{enumerate}
\end{lemma}

Before presenting our lower bound for adaptive batches, we will briefly describe the ideas behind Lemma \ref{lem:lower_fixed_batches}, which we build on. Since Lemma \ref{lem:lower_fixed_batches} itself builds on lower bounds in the fully sequential setting \citep{Scarlett17}, we will briefly discuss those first.

In the fully sequential setting, the idea is to consider a ``needle in haystack'' problem \citep{Scarlett17, Cai21}, where we construct a class of hard-to-distinguish functions $\mathcal{F}=\{f_1,\ldots,f_M\}$ with maximal value dictated by a parameter $\epsilon>0$, with the properties that $\lVert f_i\rVert_k \leq \Psi$ for each $i\in[M]$, and that any $x\in\mathcal{X}$ can be $\epsilon$-optimal for at most one function. Then, achieving simple regret $r^{(T)} \leq \epsilon$ is equivalent to correctly identifying the function index $i\in[M]$. An example of $\mathcal{F}$ is shown in Figure \ref{fig:bump} in the appendix. \citep{Scarlett17} shows that for any given $T$, we can choose $\epsilon = \Theta(\sqrt{M/T})$ in the construction of $\mathcal{F}$.

With this understanding, Lemma \ref{lem:lower_fixed_batches} aims to identify a ``bad'' batch (say the $i$-th batch) that starts too early and lasts too long, i.e., $t_{i-1} < T_{i-1} $ and $t_i \geq T_i$ for suitably-defined $T_{i-1}$ and $T_i$. 
To see why, notice that if the $i$-th batch is bad, then we can construct $\mathcal{F}$ with $\epsilon = \Theta(\sqrt{M/T_{i-1}})$ such that $r^{(t_{i-1})}\geq \epsilon$, which implies $R^i \geq (t_i-t_{i-1}) \epsilon > (T_i-T_{i-1})\epsilon$. Since $R^i$ is entirely determined by $T_{i-1}$ and $T_i$, it follows that suitably choosing $T_i$ would lead to the same order of lower bound for each $R^i$. It is then argued that for any fixed batches specified by $t_1, \ldots, t_B$, there will always be one bad batch (say the $i$-th batch) such that both $t_{i-1} < T_{i-1} $ and $t_i \geq T_i$ hold, which incurs a lower bound on $R^i$, and consequently a lower bound on $R_T$.

\subsection{Extension to Adaptive Batches}
We are now ready to present our novel lower bounds for adaptive batch sizes. In the following result, the expectation $\mathbb{E}$ is taken over the sampling noise and any randomness in the algorithm.
\begin{theorem}
    \label{thm:lower_adaptive_batches}
    Under the setup of Section \ref{sec:setup_standard},
    for any batched algorithm with a constant or growing number $B$ of adaptive batches, there exists $f$ with $ \lVert f\rVert_k \leq \Psi $ such that 
    \begin{itemize}
        \item for the SE kernel, it holds that $\mathbb{E}[R_T] =\Omega\Bigl(B^{-2}T^{\frac{1-\eta}{1-\eta^B}}(\log T)^{\frac{d\eta(\eta-\eta^B)}{1-\eta^B}} \Bigr)$, where $\eta=\frac{1}{2}$;
        \item for the Mat\'ern kernel,  it holds that $\mathbb{E}[R_T] =\Omega\Bigl(B^{-(2\eta+1)}T^{\frac{1-\eta}{1-\eta^B}} \Bigr)$, where $\eta=\frac{\nu}{2\nu+d}$.
    \end{itemize}
\end{theorem}

\begin{remark}
    \label{rmk:necessary_B}
    Using Theorem \ref{thm:lower_adaptive_batches}, we can deduce that it is necessary to have $B \geq (\log_{1/\eta}\log_2T )(1-o(1))$ batches in order to attain regret matching $\Omega^*(\sqrt{T\gamma_T})$ to within $\log T$ factors:
    
    \begin{itemize}
        \item If we choose $B=\lceil\log_{1/\eta}\log_2T \rceil $, 
        then $B = c+\log_{1/\eta}\log_2 T $ for some $c\in[0,1)$,
        and  $\eta^B = \eta^c/ \log_2 T $, which means
        \begin{align}
            T^{\frac{1-\eta}{1-\eta^B}}
            & = T^{1-\eta}T^{\frac{(1-\eta)\eta^B}{1-\eta^B}} \label{eq:T_step1} \\
            & = T^{1-\eta}T^{(1-\eta)\frac{\eta^c}{(\log_2 T)-\eta^c}} \\
            & > T^{1-\eta}T^{(1-\eta)\frac{\eta}{(\log_2 T)-\eta}} \label{eq:T_step3} \\
            & = T^{1-\eta}2^{(1-\eta)\eta\frac{\log_2 T}{(\log_2 T)-\eta}} \\
            & = T^{1-\eta}2^{\Theta(1)},
        \end{align}
        and similar steps give $(\log T)^{\frac{d(\eta-\eta^B)}{2(1-\eta^B)}} = (\log T)^{\frac{d\eta}{2}}2^{o(1)} $.  Accordingly, the regret bound has $T^{1-\eta}$ scaling to within logarithmic factors; substituting $\eta$ gives that this matches $\sqrt{T\gamma_T}$. 
        \item If we choose $B=\lceil c'\log_{1/\eta} \log_2T \rceil$ where $c'\in(0,1)$, 
        then $B= c+c'\log_{1/\eta}(\log_2T) $ for some $c\in[0,1)$,
        and $\eta^B= \eta^{c} (\log_2 T)^{-c'}$, so similar steps to \eqref{eq:T_step1}--\eqref{eq:T_step3} yield
        \begin{align}
            T^{\frac{1-\eta}{1-\eta^B}}
            & > T^{1-\eta}T^{(1-\eta)\frac{\eta}{(\log_2 T)^{c'}-\eta }}.
        \end{align}
        We observe that the second factor $T^{(1-\eta)\frac{\eta}{(\log_2 T)^{c'}-\eta }}$ 
        scales as $T^{o(1)}$ but $(\log T)^{\omega(1)}$.
        Moreover, we still have $(\log T)^{\frac{d(\eta-\eta^B)}{2(1-\eta^B)}} = (\log T)^{\frac{d\eta}{2}}2^{o(1)} $.
        Thus, for both kernels, the regret bound has an extra $(\log T)^{\omega(1)}$ term, so is not optimal to within $\log T$ factors.
    \end{itemize}
    Recall from Section \ref{sec:standard_upper_bound} that it is sufficient to have $B \leq \lceil \log_{1/a}\log_2 T \rceil + 1$ batches to recover the near-optimal $O^*(\sqrt{T\gamma_T})$ regret. Hence, the batch sizes in \eqref{eq:grow_b_batch_sizes} yield an optimal number of batches to within $1+o(1)$
    when $a$ is chosen sufficiently close to $\eta$.
\end{remark}

\begin{remark} \label{rem:compare_lem}
    Compared with Lemma \ref{lem:lower_fixed_batches}, Theorem \ref{thm:lower_adaptive_batches} has an extra dependence on $B$ that is inverse polynomial, which is consistent with the fact that we consider a more powerful class of algorithms.
    However, we are justified to restrict $B=O(\log\log T)$,
    because (i) this order of $B$ is sufficient for near-optimal regret by Theorem \ref{thm:upper_grow_b}, and (ii) for any larger $B$, we can use the general lower bound in \citep{Scarlett17} that holds regardless of $B$.  
    For $B=O(\log\log T)$, the impact of the $B$ term in Theorem \ref{thm:lower_adaptive_batches} is at most poly-loglog, which is very mild. 
    This means that adaptive batches, despite being conceptually more flexible, have essentially the same minimax behavior as fixed batches.  We also note that this formalizes an \emph{intuitive but informal} argument in \citep{Li22} that adaptivity shouldn't help, which is based on recursively considering the regret incurred based on when the previous batch ended.
\end{remark}

The full proof of Theorem \ref{thm:lower_adaptive_batches} is provided in Appendix \ref{app:proof_lower_adaptive_batches}, which follows from the discussion in Section \ref{sec:adaptive_batches_preliminary} (but the details are nearly all distinct) and is also inspired by the batched multi-armed bandit problem \citep{Gao19}. It is outlined as follows.

Recall that for fixed batches, \emph{we know which batch is bad beforehand}, i.e., $t_{i-1}<T_{i-1} $ and $t_i\geq T_i$ for some $i$. Therefore, we can construct $\mathcal{F}$ with $\epsilon = \Theta(\sqrt{M/T_{i-1}})$, which is difficult for the $i$-th batch, so that the regret incurred on this batch suffices to establish the desired lower bound. 
However, this argument breaks down for adaptive batches, as \emph{we don't know which batch is bad beforehand}, so no single $\mathcal{F}$ (or $\epsilon$) can handle all possible cases, and we need to do a ``full'' analysis specific to cumulative regret that is not based on a reduction to any existing result.

In more detail,
we consider $B$ families of hard functions $\{\mathcal{F}_i\}_{i=1}^B$, where each $\mathcal{F}_i$ is constructed to handle the case that the $i$-th batch is bad, and is parameterized by $\epsilon_i$. 
The proof starts by lower bounding the regret incurred between $T_{i-1} $ and $T_i$ for functions from $\mathcal{F}_i$. The lower bound is expressed in terms of probabilities conditioned on functions from $\mathcal{F}_i$, making it difficult to combine with those from any other class $i' \neq i$. 
Therefore, we use
$\mathbb{P}_{f}(E\cap A) + \mathbb{P}_{0}(E^c\cap A) \geq \int_A \min\{d\mathbb{P}_f, d\mathbb{P}_0\}$,
where $f \in \mathcal{F}_i$, $\mathbb{P}_0 $ denotes probabilities under the all-zero function, and $E,A$ are some events, 
to change measure from $\mathbb{P}_{f}$ to $\mathbb{P}_{0}$. Lower bounding the integral introduces total variation (TV) terms, which are upper bounded in Lemma \ref{lem:aux_avg_tv_bound}.
Then, by suitably choosing $\epsilon_i$ for each $i$, we can aggregate the lower bounds from each batch (in particular using $\sum_i \mathbb{P}_0(A_i) \geq 1$ for some ``bad events'' $\{A_i\}_{i=1}^B$). This leads to a lower bound that is averaged from all hard functions, which then lower bounds the worst-case regret.

Based on Theorem \ref{thm:lower_adaptive_batches}, we are also able to derive high probability bounds on $R_T$ that are of the same order.

\begin{corollary}
    \label{cor:lower_adaptive_batches}
    Under the setup of Section \ref{sec:setup_standard},
    for any batched algorithm with a constant or growing number $B$ of adaptive batches, and any $\delta \in (0,\frac{1}{11B})$, there exists $f$ with $ \lVert f\rVert_k \leq \Psi $ such that with probability at least $\delta$, we have
    \begin{itemize} 
        \item for the SE kernel, $R_T = \Omega\Bigl(B^{-2}T^{\frac{1-\eta}{1-\eta^B}}(\log T)^{\frac{d\eta(\eta-\eta^B)}{1-\eta^B}} \Bigr)$, where $\eta=\frac{1}{2}$.
        \item for the Mat\'ern kernel, $R_T =\Omega\Bigl(B^{-(2\eta+1)}T^{\frac{1-\eta}{1-\eta^B}} \Bigr)$, where $\eta=\frac{\nu}{2\nu+d}$.
    \end{itemize}    
\end{corollary}

\begin{remark}
    A limitation of Corollary \ref{cor:lower_adaptive_batches} is that it only applies for $\delta < \frac{1}{11 B}$, whereas ideally one would show a lower bound that holds for fixed $\delta > 0$ not depending on $B$.  However, since we are interested in $B = O(\log \log T)$ (see Remark \ref{rem:compare_lem}), this corollary still places a strong restriction on how small the error probability can be made.  We expect the same lower bounds (in the $\Omega^*(\cdot)$ sense) to also hold for constant $\delta$, but we do not attempt to prove this.
\end{remark}
The proof of Corollary \ref{cor:lower_adaptive_batches} is given in Appendix \ref{app:proof_cor}, where the main idea is based on \citep{Scarlett17} and uses a reverse Markov inequality on the cumulative regret incurred in a specific range.

\section{Robust Setting}
\label{sec:robust}

In the robust setting (see Section \ref{sec:setup_robust}), we again assume the domain is finite, with the understanding that extending the results to $\mathcal{X}=[0,1]^d$ is standard (see Remark \ref{rmk:extent_to_continuous_domain}).
Our main contribution is a batched algorithm that achieves the near-optimal $O^*(\sqrt{T\gamma_T})$ regret for growing $B$, i.e., when $B=O(\log\log T)$. We also discuss results for constant $B$ in Appendix \ref{app:robust_fixed_B_upper_bounds}.

\subsection{The Robust-BPE Algorithm}
We present the robust-BPE algorithm (Algorithm \ref{algo:rbpe}), which extends the BPE algorithm to consider adversarial perturbations. 
Similar to BPE, robust-BPE maintains a set of potential optimizers $\mathcal{X}_i$, explores uncertain points in each batch, eliminates suboptimal points from $\mathcal{X}_i$ at the end of each batch, and initializes the next batch with the updated $\mathcal{X}_i$.
The major difference between robust-BPE and BPE is that the former explores a (potentially) larger set of points in each batch. As shown in line \ref{line:set_to_explore} in Algorithm \ref{algo:rbpe}, the set under consideration is $\mathcal{X}_{\xi, i}:= \cup_{x\in\mathcal{X}_i}\{x+\delta:\delta\in\Delta_{\xi}(x) \},$\footnote{We can actually make $\mathcal{X}_{\xi, i}$ smaller to potentially improve the statistical efficiency without compromising the theoretical guarantee; this is discussed in Appendix \ref{app:simplify_robust_bpe}.} which includes $\mathcal{X}_i$ under all possible perturbations.  The elimination rule in line \ref{line:eliminate_rule} is chosen such that points with perturbed UCB lower than the maximum perturbed LCB are eliminated. 

The pseudocode for robust-BPE is shown in Algorithm \ref{algo:rbpe}, and the notations $\sigma^i_{j-1}$, $\mu^i$, $\sigma^i$, $\mathrm{UCB}_i(x)$ and $\mathrm{LCB}_i(x)$ therein have the same meaning as those in the BPE algorithm, which we define in Appendix \ref{app:bpe}. 

\begin{algorithm}[!t]
\caption{Robust-BPE for finite domain \label{algo:rbpe}}
\begin{algorithmic}[1]
    \REQUIRE Finite domain $\mathcal{X}$, GP prior $(\mu_0, \sigma_0)$, number of batches $B$, batch sizes $N_1, ..., N_B$
    \STATE $\mathcal{X}_1 \gets \mathcal{X}$
    \FOR{$i \gets 1, \dots, B$} 
        \STATE $\mathcal{S}_i\gets\emptyset$
        \STATE $\mathcal{X}_{\xi, i}\gets \cup_{x\in\mathcal{X}_i}\{x+\delta:\delta\in\Delta_{\xi}(x) \} $ \label{line:set_to_explore}
        \FOR{$j \gets 1, 2, \dots, N_i$}
            \STATE Compute $\sigma^i_{j-1}$ only based on $\mathcal{S}_i$
            \STATE $x_t \gets \argmax_{x\in \mathcal{X}_{\xi, i}} \sigma^i_{j-1}(x)$
            \STATE $\mathcal{S}_i\gets \mathcal{S}_i\cup\{x_t\}$
        \ENDFOR
        \STATE Observe all points in $\mathcal{S}_i$
        \STATE Compute $\mu^{i}$ and $\sigma^{i}$ only based on $\mathcal{S}_i$
        \STATE Compute $\mathrm{UCB}_i(x)$ and $\mathrm{LCB}_i(x)$ for $x\in\mathcal{X}_i$
        \STATE $\mathcal{X}_{i+1} \gets \{ x \in \mathcal{X}_{i}: \min_{\delta\in\Delta_{\xi}(x)} \mathrm{UCB}_i(x+\delta) \geq \max_{x\in{\mathcal{X}_i}} \min_{\delta\in\Delta_{\xi}(x)} \mathrm{LCB}_i(x+\delta) \}$ \label{line:eliminate_rule}
    \ENDFOR   
\end{algorithmic}
\end{algorithm}

\subsection{Regret Analysis for Robust-BPE}
Similar to BPE, the performance of robust-BPE crucially depends on the batch sizes, and using the batch sizes from \eqref{eq:grow_b_batch_sizes} leads to the following result.

\begin{theorem}
    \label{thm:upper_robust}
    Under the setup of Section \ref{sec:setup_robust},
    using batch sizes defined in \eqref{eq:grow_b_batch_sizes} with $a\in(\frac{1}{2},1)$ for either the SE kernel or the Mat\'ern kernel,
    for any $\delta \in (0,1)$,
    robust-BPE yields with probability at least $1-\delta$ that
    \begin{align*}
    	R_{\xi,T}
        = O \Bigl(\bigl(\Lambda + \sqrt{\log \log\log T}\bigr)\sqrt{T\gamma_T}\Bigr),
    \end{align*}
    where $\Lambda=\Psi+\sqrt{\log(|\mathcal{X}|/\delta)}$.
    In addition, 
    with $a\in(\frac{\nu}{2\nu+d},\frac{1}{2}]$ for the Mat\'ern kernel,
    robust-BPE yields with probability at least $1-\delta$ that
    \begin{align*}
    	R_T
        = O \Bigl(\bigl(\Lambda + \sqrt{\log \log\log T}\bigr)\sqrt{T\overline{\gamma}_T}\Bigr),
    \end{align*}
    where $\overline{\gamma}_T$ is any ``well-behaved'' upper bound on $\gamma_T$ in the sense of Theorem \ref{thm:upper_grow_b} (see also Definition \ref{def:gamma_t} in the appendix).
\end{theorem}
\begin{remark}
    Compared to Theorem \ref{thm:upper_grow_b}, Theorem \ref{thm:upper_robust} has the same order of regret, which implies the robust setting has similar minimax regret as the standard setting.
\end{remark}
\begin{remark}
    \label{rmk:robust_simple_regret}
    For the simple regret, it is straightforward to find $x^{(T)}$ such that $r_{\xi}(x^{(T)}) \leq R_{\xi,T}/T$ with high probability. Specifically, we can choose $x^{(T)} = x_T$,
    which guarantees $r_{\xi}(x^{(T)}) \leq R_{\xi,T}/T $ with probability at least $1-\delta$;
    this is because our upper bound on $r_{\xi}(x)$ is non-increasing over time (Lemma \ref{lem:aux_xi_regret}), and our upper bound on $R_{\xi,T}$ is obtain by aggregating these bounds across times $t=1,\ldots, T$.
    Then, from Theorem \ref{thm:upper_robust} we can upper bound the simple regret by $r_{\xi}(x^{(T)})=O^*(\sqrt{\gamma_T/T})$.
    By comparison, \citep{Bogunovic18} upper bounds the simple regret by $O^*(\gamma_T/\sqrt{T})$ with a fully sequential UCB-like algorithm.
    Thus, our bound removes an unnecessary $\sqrt{\gamma_T}$ factor in analogy with the developments of the standard setting outlined in Section \ref{sec:related}. 
\end{remark}
\begin{remark}
    Additionally, \citep{Cai21} presents lower bounds for simple regret in this robust setting, which are $r_{\xi}(x^{(T)}) = \Omega(\sqrt{(\log T)^{d/2}/T})$ for the SE kernel and $r_{\xi}(x^{(T)}) = \Omega(T^{\frac{-\nu}{2\nu+d}})$ for the Mat\'ern kernel. Substituting $\gamma_T$ in our upper bounds shows $r_{\xi}(x^{(T)}) = O^*(\sqrt{(\log T)^{d+1}/T})$ for the SE kernel, and $r_{\xi}(x^{(T)}) = O^*(T^{\frac{-\nu}{2\nu+d}}(\log T)^{\frac{\nu}{2\nu+d}})$ for the Mat\'ern kernel. It follows that our upper bounds match the lower bounds up to logarithmic factors.
\end{remark}

The full proof of Theorem \ref{thm:upper_robust} is given in Appendix \ref{app:proof_upper_robust}. While the proof builds on that of Theorem \ref{thm:upper_grow_b} in Appendix \ref{app:proof_thm_upper_grow_b} and that of Lemma \ref{lem:upper_grow_b_li} in \citep{Li22}, some of the key steps are substantially different.
The proof is sketched as follows.
First, by choosing the same parameter $\beta$ as the BPE algorithm (see Appendix \ref{app:bpe}),
we have valid confidence intervals $f(x)\in[\mathrm{LCB}_i(x),\mathrm{UCB}_i(x)]$ for all $x\in \mathcal{X}$ and all $i\in[B]$ with probability at least $1-\delta$ (Lemma \ref{lem:aux_conf}).
Then, we show that an $\xi$-robust optimal point will never be eliminated (Lemma \ref{lem:aux_robust_elimination}), and that at the end of the $(i-1)$-th batch, the posterior variances of points in $\mathcal{X}_{\xi, i}$ are uniformly upper bounded by $O(\sqrt{\gamma_{N_{i-1}}/N_{i-1}})$ (Lemma \ref{lem:aux_robust_posterior_variance}). 
With these two facts, it follows that $\xi$-regret in the $i$-th batch is upper bounded by $O(\sqrt{(\beta \gamma_{N_{i-1}})/ N_{i-1} } )$ (Lemma \ref{lem:aux_xi_regret}), which leads to $R_{\xi}^i= O (N_{i}\sqrt{(\beta \gamma_{N_{i-1}})/ N_{i-1} } ) $. 
The final steps are identical to those for Theorem \ref{thm:upper_grow_b}, where we show 
the upper bound of $R_{\xi,T}$ is of the same order as that of $R_{\xi}^B$.

\section{Conclusion}
In this paper, we have refined and extended regret bounds for batched kernelized bandits.  In particular, we established the optimal number of batches (to within $1+o(1)$) for near-optimal regret, removed a $B$ term in the regret bound of \citep{Li22}, derived a lower bound that holds even for adaptive batches, and gave an extension to the robust setting that leads to improved simple regret compared to \citep{Bogunovic18}.

\clearpage
\bibliographystyle{apalike}
\bibliography{references}

\clearpage
\appendix
\section{The BPE Algorithm}
\label{app:bpe}

In this section we present the BPE algorithm \citep{Li22} and some relevant facts.
First, we note that while the theorem is for $f$ (the black-box function to optimize) lying in some RKHS, the BPE algorithm models $f$ as though it were drawn from a Gaussian process (GP) with prior mean $\mu_0=0$ and kernel function $k(\cdot,\cdot)$.
For a sequence of sampled points $(x_1,\ldots, x_t)$ and their noisy observations $\mathbf{y}_t = (y_1,\ldots, y_t)$, the posterior distribution is also a GP, with posterior mean $\mu_t(\cdot)$ and posterior variance $\sigma_t(\cdot)^2$ given by:
\begin{align}
    \mu_t(x) & = \mathbf{k}_t(x)^T(\mathbf{K}_t+\lambda \mathbf{I}_t)^{-1}\mathbf{y}_t, \\
    \sigma_t(x)^2 & = k(x,x) - \mathbf{k}_t(x)^T(\mathbf{K}_t+\lambda \mathbf{I}_t)^{-1}\mathbf{k}_t(x),
\end{align}
where $\mathbf{k}_t(x) = [k(x_i,x)]_{i=1}^t $, $\mathbf{K}_t = [k(x_t,x_{t'})]_{t,t'} $, and $\lambda>0$ is a free parameter we choose to be $\sigma^2$ to simplify the expression for $\beta$ (which only affect constant factors).
Next, we introduce some notations used by BPE:
\begin{itemize}
    \item $\sigma^i_{j-1}(\cdot)^2$ is the GP posterior variance computed using only the first $j-1$ points sampled in the $i$-th batch, and we define $\sigma^i_{0}(\cdot)^2$ to be the prior variance $\sigma_0(\cdot)^2$. Moreover, $\sigma^i(\cdot)^2$ is the posterior variance computed using all $N_i$ points sampled in the $i$-th batch, and similarly $\mu^i(\cdot)$ is the posterior mean using those points. 
    \item $\mathrm{UCB}_i(x)$ and $\mathrm{LCB}_i(x)$ are upper and lower confidence bounds for the $i$-th batch, where
    \begin{align}
        \mathrm{UCB}_i(x) & = \mu^{i}(x) + \sqrt{\beta}\sigma^{i}(x), \\
        \mathrm{LCB}_i(x) & = \mu^{i}(x) - \sqrt{\beta}\sigma^{i}(x),
    \end{align}
    with $\beta = \bigl( \Psi + \sqrt{2 \log \frac{|\mathcal{X}| B}{\delta}} \bigr)^2$, where $\delta \in (0,1)$.     
\end{itemize}
We present the BPE algorithm in Algorithm \ref{algo:bpe}. Note that we do not need the time horizon $T$ as an algorithmic input, because such information is already contained in the batch sizes ($\sum_{i=1}^B N_i=T$).
\begin{algorithm}[!h]
\caption{Batched Pure Exploration (BPE) for finite domain \label{algo:bpe}}
\begin{algorithmic}[1]
    \REQUIRE Finite domain $\mathcal{X}$, GP prior $(\mu_0, \sigma_0)$, number of batches $B$, batch sizes $N_1, ..., N_B$
    \STATE $\mathcal{X}_1 \gets \mathcal{X}$
    \FOR{$i \gets 1, \dots, B$} 
        \STATE $\mathcal{S}_i\gets\emptyset$
        \FOR{$j \gets 1, 2, \dots, N_i$}
            \STATE Compute $\sigma^i_{j-1}$ only based on $\mathcal{S}_i$
            \STATE $x_t \gets \argmax_{x\in \mathcal{X}_{i}} \sigma^i_{j-1}(x)$
            \STATE $\mathcal{S}_i\gets \mathcal{S}_i\cup\{x_t\}$
        \ENDFOR
        \STATE Observe all points in $\mathcal{S}_i$
        \STATE Compute $\mu^{i}$ and $\sigma^{i}$ only based on $\mathcal{S}_i$
        \STATE Compute $\mathrm{UCB}_i(x)$ and $\mathrm{LCB}_i(x)$ for $x\in\mathcal{X}_i$
        \STATE $\mathcal{X}_{i+1} \gets \{ x \in \mathcal{X}_{i}: \mathrm{UCB}_i(x) \geq \max_{x\in\mathcal{X}_i} \mathrm{LCB}_i(x) \}$
    \ENDFOR   
\end{algorithmic}
\end{algorithm}

Next we present some properties of the BPE algorithm. The first property relies on the following confidence bound.
\begin{lemma} 
    \emph{\citep[Theorem 1]{Vakili21a}}
    \label{lem:aux_conf}
    Let $f$ be a function with RKHS norm at most $\Psi$, and let ${\mu}_t(x)$ and ${\sigma}_t(x)^2$ be the posterior mean and variance based on $t$ points $\mathrm{X} = (x_1,\dotsc,x_t)$ with observations $\mathrm{y} = (y_1,\dotsc,y_t)$.  Moreover, suppose that the $t$ points in $\mathrm{X}$ are chosen independently of all samples in $\mathrm{y}$.  Then, for any fixed $x\in \mathcal{X}$ and any $\delta \in (0,1)$, it holds with probability at least $1-\delta$ that $|f(x)-\mu_t(x)|\leq\sqrt{\beta(\delta)}\sigma_t(x)$, where $\beta(\delta) = \bigl(\Psi + \sqrt{2 \log \frac{1}{\delta}}\bigr)^2$.
\end{lemma}

\begin{proposition}
    \label{prop:bpe_confidence_bounds}
    For any $\delta \in (0,1)$, running the BPE algorithm with $\beta = \bigl( \Psi + \sqrt{2 \log \frac{|\mathcal{X}| B}{\delta}} \bigr)^2$, it holds with probability least $1-\delta$ that $f(x)\in[\mathrm{LCB}_i(x),\mathrm{UCB}_i(x)]$ for all $x\in\mathcal{X}$ and for all $i\in[B]$, i.e., we have valid confidence bounds in all batches.
\end{proposition}
\begin{proof}
    The result follows by using Lemma \ref{lem:aux_conf} with a union bound over all $\lvert \mathcal{X}\rvert$ points and all $B$ batches. 
\end{proof}

\begin{proposition}
    \label{prop:property_bpe}
    Assuming valid confidence bounds in all batches,
    running the BPE algorithm yields that
    \begin{enumerate}
        \item $R^1 \leq 2 N_1\sqrt{\beta}$;
        \item For $2 \leq i \leq B$, $R^i \leq 2\sqrt{C_1\beta} \cdot N_i\sqrt{\gamma_{N_{i-1}}/N_{i-1}} $;
    \end{enumerate}
    where $C_1 = \frac{8}{\log(1+\sigma^{-2})}$.
\end{proposition}
\begin{proof}
    The first point is equation (7) in \citep{Li22}.
    The second point follows from Lemma 4 in \citep{Li22}, which states that for any $x\in\mathcal{X}_i$ where $2\leq i\leq B$, we have $f(x^*) - f(x) \leq 2 \sqrt{\frac{C_1\gamma_{N_{i-1}}\beta}{N_{i-1}}}$, which implies
    $R^i \leq  2N_i \sqrt{\frac{C_1\gamma_{N_{i-1}}\beta}{N_{i-1}}}$.
\end{proof}

\section{Proof of Main Results}
\label{app:proof}
\subsection{Proof of Theorem \ref{thm:upper_grow_b} (Standard Setting Upper Bounds)}
\label{app:proof_thm_upper_grow_b}

We first present the properties of the batch sizes defined in \eqref{eq:grow_b_batch_sizes}, then present the regret analysis.

\begin{proposition}
    \label{prop:grow_b_batch_sizes_property}
    Consider the batch sizes $N_i = \min \bigl\{\lceil T^{1-a^i} \rceil, T-\sum_{j=1}^{i-1}N_j\bigr\}$ where $a\in(0,1)$ and $i\in [B]$. 
    Then we have $B = \bigl(\log_{1/a}\log T\bigr) (1+o(1)) $.
    Specifically,
    \begin{align}
        B & \leq \lceil \log_{1/a}\log_2 T \rceil + 1, \label{eq:B_upper_bound} \\
        B & \geq \Bigl\lfloor \log_{1/a}\Bigl(\frac{\log T}{\log\log T}\Bigr) \Bigr\rfloor. \label{eq:B_lower_bound}
    \end{align}
\end{proposition}

\begin{proof}
    \textbf{Proof of \eqref{eq:B_upper_bound}.} 
    Define 
    $\tilde{B} = \lceil \log_{1/a}\log_2 T \rceil$. Then,
    \begin{align}
        \lceil T^{1-a^{\tilde{B}}} \rceil
        \geq T^{1-a^{\tilde{B}}}
        \geq T^{1-a^{\log_{1/a}\log_2 T}}
        = \frac{T}{2},
    \end{align}
    which means
    \begin{align}
        \sum_{i=1}^{\tilde{B}+1} \lceil T^{1-a^{i}} \rceil 
        \geq \lceil T^{1-a^{\tilde{B}}} \rceil + \lceil T^{1-a^{\tilde{B}+1}} \rceil 
        \geq T,
    \end{align}
    which implies $B\leq \tilde{B}+1 = \lceil \log_{1/a}\log_2 T \rceil + 1 $.
    
    \textbf{Proof of \eqref{eq:B_lower_bound}.} We drop the ceiling operation by $\lceil T^{1-a^i} \rceil \leq 2 T^{1-a^i}$, and upper bound $N_{B} $ by $\lceil T^{1-a^{B}} \rceil$.
    Notice $T^{1-a^i} \leq T^{1-a^{B}} $, so $2 \sum_{i=1}^{B} T^{1-a^i} \leq 2 B T^{1-a^{B}}  $. 
    Then, we will find $B$ such that $2 B T^{1-a^{B}} \leq T$, which by taking the logarithm on each side is equivalent to
    \begin{align}
        a^{B}\log T\geq \log 2B. \label{eq:lower_B}
    \end{align}
    Now we claim that $B = \bigl\lfloor \log_{1/a}\bigl(\frac{\log T}{\log\log T}\bigr) \bigr\rfloor$ satisfies \eqref{eq:lower_B}. To ease notation, we let $B' = \log_{1/a}\bigl(\frac{\log T}{\log\log T}\bigr) \geq B$. Then, substituting $B$, the LHS of \eqref{eq:lower_B} becomes
    \begin{align}
        a^{B}\log T
        \geq a^{B'}\log T
        = (1/a)^{-B'}\log T
        = \frac{\log \log T}{\log T}\cdot \log T
        = \log\log T,
    \end{align}
    and the RHS of \eqref{eq:lower_B} becomes
    \begin{align}
        \log 2B
        \leq \log 2B'
        & = \log 2 + \log\log_{1/a}\Bigl(\frac{\log T}{\log\log T}\Bigr) \\
        & = \log 2 + \log \frac{\log\log T - \log\log\log T}{\log (1/a)} \\
        & = \log \frac{2}{\log(1/a)} + \log (\log\log T - \log\log\log T) \\
        & \leq \log\log\log T + O(1),
    \end{align}
    which is smaller than $\log\log T$ for sufficiently large $T$. 
\end{proof}

From Remark \ref{rmk:choose_a_and_more} stated after Theorem \ref{thm:upper_grow_b} we know that it is desirable to choose small $a$. For the SE kernel, we want $a>\frac{1}{2}$ and be close to $\frac{1}{2}$, which means $\log \frac{2}{\log(1/a)}$ is close to $\log2 - \log\log 2 < 1.06 $, which is a very small constant (and \eqref{eq:lower_B} is satisfied for $T\geq 3$). For the Mat\'ern kernel, $\log \frac{2}{\log(1/a)}$ is close to $\log 2 - \log\log (2+\frac{d}{\nu})$, which is smaller. This means the conditions for \eqref{eq:lower_B} to hold are very mild.

\textbf{Proof of Theorem \ref{thm:upper_grow_b}.} 
We first give the proof for both kernels when $a\in(\frac{1}{2}, 1)$, then give the proof for the Mat\'ern kernel when $a\in (\frac{\nu}{2\nu+d}, \frac{1}{2}] $.

\emph{Case $1$: Both kernels, $a\in(\frac{1}{2}, 1)$.}
For the first batch, we have
\begin{align}
    R^1 
    \leq 2N_1\sqrt{\beta}
    \leq 4T^{1-a}\sqrt{\beta}
    = O\bigl(T^{1-a} \sqrt{\beta}\bigr),
\end{align}
where the first inequality follows from Proposition \ref{prop:property_bpe}, and the second follows since
$N_1 \leq \lceil T^{1-a^i} \rceil \leq 2T^{1-a^i} $.
For any remaining batches where $i = 2,\ldots, B$, 
by the choice of batch sizes in \eqref{eq:grow_b_batch_sizes}, we have $N_{i} \leq \lceil T^{1-a^{i}} \rceil \leq 2T^{1-a^{i}}$ and $N_{i-1} = \lceil T^{1-a^{i-1}} \rceil \geq T^{1-a^{i-1}} $.
Then, 
\begin{align}
    R^i
    & \leq 2\sqrt{C_1\beta} N_i(N_{i-1})^{-\frac{1}{2}}(\gamma_{N_{i-1}})^{\frac{1}{2}} \label{eq:use_bpe_property} \\
    & \leq 4\sqrt{C_1\beta} T^{1-a^i} (T^{1-a^{i-1}})^{-\frac{1}{2}}(\gamma_{N_{i-1}})^{\frac{1}{2}} \\
    & \leq 4\sqrt{C_1\beta} T^{1-a^i} (T^{1-a^{i-1}})^{-\frac{1}{2}}(\gamma_T)^{\frac{1}{2}} \label{eq:use_gamma_t_mon} \\
    & = O\Bigl( T^{a^{i-1}(\frac{1}{2}-a)}  \sqrt{\beta T\gamma_T} \Bigr),
\end{align}
where \eqref{eq:use_bpe_property} uses Proposition \ref{prop:property_bpe}
and \eqref{eq:use_gamma_t_mon} follows since $\gamma_t$ is increasing. 
Now, for every $i\in[B]$ we define $U_i= T^{a^{i-1}(\frac{1}{2}-a)}  \sqrt{\beta T\gamma_T} $ so that $R^i=O(U_i)$. Notice that for $2\leq i \leq B$, we have
\begin{align}
    \frac{U_{i}}{U_{i-1}}
    = \frac{T^{a^{i-1}(\frac{1}{2}-a)} }{T^{a^{i-2}(\frac{1}{2}-a)}}
    = T^{a^{i-2}(a-1)(\frac{1}{2}-a)}
    := T^{a^{j}b},
\end{align}
where $j=i-2$ and $b = (a-1)(\frac{1}{2}-a)\in (0,\frac{1}{16}]$ (since $a \in \big(\frac{1}{2},1\big)$). Recall that $B \leq \lceil\log_{1/a}\log_2T \rceil+1 $ by Proposition \ref{prop:grow_b_batch_sizes_property}, so $j \leq \lceil\log_{1/a}\log_2T \rceil -1 \leq \log_{1/a}\log_2T$. This means $a^j \geq (\log_2 T)^{-1} $, which implies $ U_i/U_{i-1} =  T^{a^{j}b} = 2^{a^{j}b \log_2 T} \geq 2^{b}$. Notice that $2^b > 1$ as $b$ is fixed and strictly positive, which means the ratio $U_i/U_{i-1} $ increases exponentially with batch index $i$, and thus $U_B$ dominates the other terms:
\begin{align}
    \sum_{i=1}^B U_i 
    \leq U_B \sum_{i=1}^B 2^{-b(B-i)}
    = U_B \sum_{k=0}^{B-1} 2^{-bk}
    \leq U_B \sum_{k=0}^{\infty} 2^{-bk}
    = \frac{2^b}{2^b-1}U_B . \label{eq:ub_dominate}
\end{align}

To upper bound $U_B$, recall from Proposition \ref{prop:grow_b_batch_sizes_property} that $B \leq \lceil \log_{1/a}\log_2 T \rceil + 1$, which means
\begin{align}
    T^{a^{B-1}(\frac{1}{2}-a)}
    \leq T^{a^{\lceil \log_{1/a}\log_2 T \rceil}(\frac{1}{2}-a)}
    \leq T^{a^{2 \log_{1/a}\log_2 T}(\frac{1}{2}-a)}
    = T^{(\log_2 T)^{-2}(\frac{1}{2}-a)}
    = 2^{(\frac{1}{2}-a) (\log_2 T)^{-1}}
    = \Theta(1), \label{eq:ub_bound}
\end{align}
where in the first two inequalities we \emph{upper bound} the exponent to $a$ because $a \in \big(\frac{1}{2},1\big)$ and $\frac{1}{2}-a < 0$.
Also notice that from the choice of $\beta$ in Proposition \ref{prop:bpe_confidence_bounds} we have
\begin{align}
    \sqrt{\beta} 
    = \Psi + \sqrt{2\log\frac{\lvert\mathcal{X}\rvert B}{\delta}} 
    = \Theta\biggl(\Psi + \sqrt{\log\frac{\lvert\mathcal{X}\rvert}{\delta}}+\sqrt{\log\log\log T}\biggr)
    = \Theta \bigl(\Lambda + \sqrt{\log \log\log T}\bigr), \label{eq:beta_bound}
\end{align}
where $\Lambda=\Psi+\sqrt{\log(|\mathcal{X}|/\delta)}$.
Finally, combining the bounds in \eqref{eq:ub_dominate}, \eqref{eq:ub_bound}, \eqref{eq:beta_bound}, we have
\begin{align}
    R_T 
    & = \sum_{i=1}^B R^i
    = \sum_{i=1}^B O(U_i)
    = O(U_B) \\
    & = O\bigl( T^{a^{B-1}(\frac{1}{2}-a)}  \sqrt{\beta T\gamma_T} \bigr) \\
    & = O\Bigl( \sqrt{\beta T\gamma_T} \Bigr) \\
    & = O\Bigl(\bigl(\Lambda + \sqrt{\log \log\log T}\bigr) \sqrt{T\gamma_T}\Bigr).
\end{align}

\emph{Case $2$: Mat\'ern kernel, $a\in (\frac{\nu}{2\nu+d}, \frac{1}{2}] $.}
To simplify notation, we define $\eta = \frac{\nu}{2\nu+d}$. We generalize the sufficient condition on $\overline{\gamma}_t$ from the theorem statement as follows.
\begin{definition}
    \label{def:gamma_t}
    Given our choice of batch sizes $N_1,N_2,\dotsc,N_B$ (dependent on the time horizon $T$), we let $\overline{\gamma}_t$ denote any sequence that serves as a ``well-behaved upper bound'' on $\gamma_t$ in the following sense:
    \begin{itemize}
        \item $\overline{\gamma}_t = c' \cdot t^{\frac{d}{2\nu+d}} \cdot\alpha_t$ for some constant $c'$ and function $\alpha_t$, where $\alpha_t$ is any function that satisfies $\alpha_{N_i}= O(\alpha_T)$ for all $i\in[B]$.
        \item $\overline{\gamma}_t$ upper bounds $\gamma_t$ for all sufficiently large $t$.
    \end{itemize}
\end{definition}

\begin{remark}
    \label{rmk:near-optimal}
    We note the following:
    \begin{itemize}
        \item The first of the above conditions is satisfied if $\alpha_t$ a non-decreasing sequence, as stated in Theorem \ref{thm:upper_grow_b}.  However, the condition stated above is strictly more general; for example, it remains true if $\alpha_t$ takes the form $\alpha_t = (\log t)^{c''} $ for some constant $c''$ (possibly with $c'' < 0$, in which case $\alpha_t$ is decreasing rather than increasing).  Based on these examples and existing bounds such as \eqref{gamma_t_mat}, we believe that $\gamma_t$ itself could satisfy the above condition and thus permit us to set $\overline{\gamma}_t = \gamma_t$, but this appears to be difficult to prove.

        \item     Recall that $\gamma_t = O \bigl( t^{\frac{d}{2\nu+d}} \left( \log t \right)^{\frac{2\nu}{2\nu +d}} \bigr)$ for the Mat\'ern kernel in \eqref{gamma_t_mat}, which means $\alpha_t$ in Definition~\ref{def:gamma_t} is at most $(\log t)^{\frac{2\nu}{2\nu+d}}$. In addition, on the domain $\mathcal{X} = [0,1]^d$, the upper bound is $O^*\bigl(\sqrt{T\gamma_T}\bigr)$ (\cite{Li22}, Sec. 3.1) and the lower bound is $\Omega\bigl(T^{\frac{\nu+d}{2\nu+d}}\bigr)$ (\cite{Scarlett17}, Theorem 2), which implies $\gamma_T = \Omega(T^{\frac{d}{2\nu+d}} (\log T)^{-c})$, where for some dimension-independent constant $c$, and consequently implies $\alpha_T = \Omega\bigl((\log T)^{-c} \bigr)$.
    This means that even using these crude bounds, the difference between $\overline{\gamma}_T$ and $\gamma_T$ at most amounts to multiplication or division by $(\log T)^{c + \frac{2\nu}{2\nu+d}} \leq (\log T)^{c+ 1}$, meaning the two give the same behavior in the $O^*(\cdot)$ sense. 
    \end{itemize}

\end{remark}

For the first batch, we still have $R^1 = O\bigl(T^{1-a} \sqrt{\beta}\bigr)$ by Proposition \ref{prop:property_bpe}. For any remaining batches where $2\leq i \leq B$, we upper bound $\gamma_t$ by $\overline{\gamma}_t$, and proceed as follows:
\begin{align}
    R^i
    & \leq 2\sqrt{C_1\beta} N_i(N_{i-1})^{-\frac{1}{2}}(\gamma_{N_{i-1}})^{\frac{1}{2}} \label{eq:use_prop_2} \\
    & \leq 2\sqrt{C_1\beta} N_i(N_{i-1})^{-\frac{1}{2}}(\overline{\gamma}_{N_{i-1}})^{\frac{1}{2}} \label{eq:use_overline} \\
    & = O\Bigl( \sqrt{\beta} T^{a^{i-1}(\eta-a)} T^{1-\eta} (\alpha_{N_{i-1}} )^{\frac{1}{2}} \Bigr) \\
    & = O\Bigl( \sqrt{\beta} T^{a^{i-1}(\eta-a)} T^{1-\eta} (\alpha_{T} )^{\frac{1}{2}} \Bigr) \label{eq:use_alpha_t} \\
    & = O\Bigl( T^{a^{i-1}(\eta-a)} \sqrt{\beta T\overline{\gamma}_T} \Bigr) \label{eq:use_overline_2},
\end{align}
where \eqref{eq:use_prop_2} follows from Proposition \ref{prop:property_bpe}, \eqref{eq:use_overline} follows since $\overline{\gamma}_t$ upper bounds $\gamma_t$,
\eqref{eq:use_alpha_t} follows from $\alpha_{N_{i-1}} = O(\alpha_T)$, and \eqref{eq:use_overline_2} follows from $\overline{\gamma}_T = c' \cdot T^{\frac{d}{2\nu+d}} \cdot\alpha_T$ and $\eta = \frac{\nu}{2\nu+d}$.
Then, we define $U_1 = T^{1-a} \sqrt{\beta}$ and $U_i = T^{a^{i-1}(\eta-a)} \sqrt{\beta T\overline{\gamma}_T}$ for $2\leq i\leq B$, so that $R^i=O(U_i)$ for $i\in[B]$. 
We observe that
\begin{align}
    \frac{U_{2}}{U_{1}}
    = T^{(a-1)(\eta-a)}\sqrt{c'\alpha_T}
    \geq c_1T^{(a-1)(\eta-a)}(\log T)^{-c/2}, \label{eq:u1_u2}
\end{align}
where $c'$ is the constant from Definition \ref{def:gamma_t}, $(a-1)(\eta-a) > 0$ because $a\in(\eta, 1/2]$, the inequality follows because $\alpha_T = \Omega\bigl((\log T)^{-c} \bigr)$ for some constant $c > 0$ by the second part of Remark \ref{rmk:near-optimal}, and $c_1$ is also a constant. 
Equation \eqref{eq:u1_u2} implies that $U_2$ dominates $U_1$, which combined with the previous analysis (\eqref{eq:ub_dominate} and the preceding paragraph) shows $U_B$ dominates the sum of the $B$ upper bounds:
\begin{align}
    \sum_{i=1}^B U_i 
    = U_1 + \sum_{i=2}^B U_i 
    \leq U_2\cdot o(1) + U_B \sum_{i=2}^B 2^{-b(B-i)}
    \leq U_B\cdot o(1) + U_B \sum_{i=1}^B 2^{-b(B-i)}
    \leq U_B \cdot \Theta(1). \label{eq:dominance_2}
\end{align}
Similarly to \eqref{eq:ub_bound}, we have $T^{a^{B-1}(\eta-a)} \leq \Theta(1)$. 
Combining this result with \eqref{eq:dominance_2} and \eqref{eq:beta_bound}, we have
\begin{align}
    R_T 
    & = \sum_{i=1}^B R^i
    = \sum_{i=1}^B O(U_i)
    = O(U_B) \\
    & = O\Bigl( T^{a^{B-1}(\eta-a)} \sqrt{\beta T\overline{\gamma}_T} \Bigr) \\
    & = O\Bigl( \sqrt{\beta T\overline{\gamma}_T} \Bigr) \\
    & = O\Bigl(\bigl(\Lambda + \sqrt{\log \log\log T}\bigr) \sqrt{T\overline{\gamma}_T}\Bigr) \label{eq:bound_mat}.
\end{align}

\subsection{Proof of Theorem \ref{thm:lower_adaptive_batches} and Corollary \ref{cor:lower_adaptive_batches} (Standard Setting Lower Bounds)}
\label{app:proof_lower_adaptive_batches}
To prove Theorem \ref{thm:lower_adaptive_batches} and Corollary \ref{cor:lower_adaptive_batches}, we formalize the discussion in Section \ref{sec:standard_lower_bound}. We first present the definitions and auxiliary lemmas that we need, then present the regret analysis.

\subsubsection{Definitions and Auxiliary Lemmas}
We start with some definitions.
For the set of hard-to-distinguish functions $\mathcal{F}=\{f_1, \ldots, f_M\}$ considered in Section \ref{sec:standard_lower_bound}, their precise expressions are not needed in the following discussion, and all we need are the following properties of $\mathcal{F}$ from \citep{Scarlett17}:

\begin{proposition}
    \label{prop:property_f}
    The set of hard-to-distinguish functions $\mathcal{F}=\{f_1, \ldots, f_M\}$ has the following properties:
    \begin{itemize}
        \item $\epsilon/\Psi$ is assumed to be sufficiently small so that $M \gg 1$, which will turn out to be trivially satisfied under the assumption that both $\sigma$ and $\Psi$ are positive constants;
        \item For each $f_i\in\mathcal{F}$, we have $f_i \in [-2\epsilon, 2\epsilon]$ and $\lVert f_i\rVert_k\leq \Psi$;
        \item Each $x\in\mathcal{X}$ can be $\epsilon$-optimal for at most one $f_i \in \mathcal{F}$;
        \item The number of functions $M$ satisfies
        \begin{align}
            M = \Theta\bigl( (\log(\Psi/\epsilon))^{d/2} \bigr) \label{eq:m_se}
        \end{align}
        for the SE kernel, and
        \begin{align}
            M = \Theta\bigl( (\Psi/\epsilon )^{d/\nu} \bigr) \label{eq:m_matern}
        \end{align}
        for the Mat\'ern kernel.
        \item $\{\mathcal{R}_m\}_{m=1}^{M}$ is a uniform partition of $[0,1]^d$, such that $f_m$ takes the maximum in the center of $\mathcal{R}_m$;
    \end{itemize}
\end{proposition}

The ``idealized'' behavior of the functions is depicted in Figure \ref{fig:bump}, though we note that in order to handle the SE kernel rather than only the Mat\'ern kernel, the actual functions used have steady decay to zero instead of finite support.

\begin{figure}[!t]
    \begin{center}
        \centerline{\includegraphics[width=0.4\linewidth]{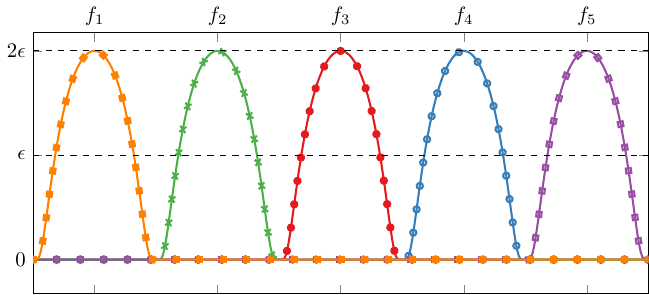}}
        \caption{Illustration of a class of hard-to-distinguish functions $\mathcal{F}$, where any $x\in\mathcal{X}$ can be $\epsilon$-optimal for at most one bump function.  This is an ``idealized'' illustration, with the actual functions used having infinite support but steady decay to zero.}
        \label{fig:bump}
    \end{center}
\end{figure}

In addition, we define the zero function $f_0\equiv0$, and introduce the following definitions:
\begin{itemize}
    \item We define $\mathbb{P}_m^t$ to be the probability measure of the observations available up to the start of time $t$ (or equivalently, the end of time $t-1$) when the underlying function is $f_m$.
    Similarly, $\mathbb{P}_0^t$ and $\mathbb{E}_0^t$ are defined with respect to the zero function $f_0$.
    For some generic function $f$ with $\lVert f\rVert_k\leq \Psi$, if $f$ is clear from the context, we omit the subscript $m$ and write $\mathbb{E}^t$ and $\mathbb{P}^t$. If $t=T$, i.e., the time horizon, we omit the superscript $T$ and write $\mathbb{E}_m$, $\mathbb{E}_0$, $\mathbb{E}$, $\mathbb{P}_m$, $\mathbb{P}_0$, and $\mathbb{P}$. 
    \item We define $P_m(\cdot|x)$ to be the conditional density function of the observation $y$ given input $x$ on function $f_m$, which is $N(f_m(x), \sigma^2)$.
\end{itemize}

In the following analysis, we apply the above definitions to $B$ classes of bump functions $\mathcal{F}_i=\{f^{(i)}_{m}\}_{m=1}^{M_i}$ where $i\in[B]$, each $\mathcal{F}_i$ has $\epsilon_i$-optimality, and has $M_i$ functions parameterized by $\epsilon_i$.\footnote{When $i$ is understood from the context, we omit the superscript to ease notation, i.e.,  we write $f_m$ instead of $ f^{(i)}_m $.
The same practice applies to other notations with subscript $m$, like $\mathbb{P}_m^t$ and $P_m(\cdot|x)$.
} 

Before proceeding, we additionally note that it suffices to prove Theorem \ref{thm:lower_adaptive_batches} for the case of algorithms whose decisions (i.e., choices of $x_t$ and choices of $t_i$) are \emph{deterministic} given the observations received so far.  This is because we are studying the average regret $\mathbb{E}[R_T]$, and for a randomized algorithm with internal randomness $U$ we can write $\mathbb{E}[R_T] = \mathbb{E}[ \mathbb{E}[R_T | U]] \ge \min_U \mathbb{E}[R_T | U]$, with the right-hand side representing a deterministic strategy.  By considering only deterministic strategies, we can note the following useful fact: Although $\mathbb{P}_m^t$ is only a measure over the observations received up to the start of time $t$, the input sequence $x_1,\dotsc,x_t$ is uniquely specified given those observations, and thus so is the regret up to (and including) time $t$.

Recall that we discussed $T_i$ in Section \ref{sec:standard_lower_bound}, which is used to determine whether a batch is bad. We now call the batch sizes specified by $T_1,\ldots, T_B$ the \emph{reference batch sizes}. Their values and properties are shown below.

\begin{proposition}
    \label{prop:reference_batches_property}
    The reference batch sizes are defined as follows:
    \begin{itemize}
        \item for the SE kernel, $T_i = \Bigl\lceil T^{\frac{1-\eta^i}{1-\eta^B}}(\log T)^{\frac{d\eta(\eta^i- \eta^B)}{1-\eta^B}}\Bigr\rceil$, where $i\in [B]$ and $\eta=\frac{1}{2}$;
        \item for the Mat\'ern kernel, $T_i = \Bigl\lceil T^{\frac{1-\eta^i}{1-\eta^B}}\Bigr\rceil$, where $i\in [B]$ and $\eta=\frac{\nu}{2\nu+d}$.
    \end{itemize}
    For convenience, we define $T_0=0$. 
    The reference batch sizes have the following properties:
    \begin{enumerate}
        \item $T_i-T_{i-1}=\Theta(T_i)$ for $i \in [B]$;
        \item $\log T_i=\Theta(\log T)$ for $i \in [B]$;
        \item For the SE kernel, $T_i(T_{i-1})^{-\eta}(\log T)^{d\eta^2}=\Theta\Bigl(T^{\frac{1-\eta}{1-\eta^B}}(\log T)^{\frac{d\eta(\eta-\eta^B)}{1-\eta^B}}\Bigr)$ for $2 \leq i \leq B$;
        \item For the Mat\'ern kernel, $T_i(T_{i-1})^{-\eta}=\Theta\Bigl(T^{\frac{1-\eta}{1-\eta^B}}\Bigr)$ for $2 \leq i \leq B$.
    \end{enumerate}
\end{proposition}
\begin{proof}
    We will prove the first two properties for the SE kernel, as the proof for the Mat\'ern kernel is entirely analogous.
    
    \textbf{Proof of the first property.}
    This holds trivially for $i=1$ since $T_0=0 $. For $i \geq 2$, 
    Notice $x\leq \lceil x\rceil \leq 2x $ when $x\geq 1$, so
    \begin{align}
        \frac{T_{i-1}}{T_{i}}
        \leq 2 T^{\frac{\eta^i-\eta^{i-1}}{1-\eta^B}} (\log T)^{\frac{d(\eta^{i-1}-\eta^i)}{1-\eta^B}}
        = 2 \Bigl(\frac{(\log T)^{d}}{T}\Bigr)^{\frac{\eta^{i-1}-\eta^i}{1-\eta^B}}.
    \end{align}
    Since $\eta$ is a fixed constant in $(0,1)$, this approaches $0$ as $T \to \infty$,
    which means $T_i$ dominates $T_{i-1}$, i.e., $T_{i-1}=o(T_i)$, so $T_i-T_{i-1}=\Theta(T_i)$.

    \textbf{Proof of the second property.} Notice that $T_i =(1+o(1)) T^{\frac{1-\eta^i}{1-\eta^B}}(\log T)^{\frac{d\eta(\eta^i- \eta^B)}{1-\eta^B}}$, and hence
    \begin{align}
        \log T_i
        & = (1+ o(1))\log\Bigl( T^{\frac{1-\eta^i}{1-\eta^B}}(\log T)^{\frac{d(\eta^i- \eta^B)}{1-\eta^B}}\Bigr) \\
        & = (1+ o(1)) \Bigl( \frac{1-\eta^i}{1-\eta^B}\log T+\frac{d(\eta^i- \eta^B)}{1-\eta^B}\log\log T\Bigr) \\
        & = (1+ o(1)) \frac{1-\eta^i}{1-\eta^B}\log T \\
        & = \Theta(\log T).
    \end{align}
    
    \textbf{Proof of the third property.} We have $x\leq \lceil x\rceil \leq 2x $ when $x\geq 1$. Then, for $i\geq 2$:
    \begin{align}
        T_i(T_{i-1})^{-\eta}(\log T)^{d\eta^2}
        & \leq 2 T^{\frac{(1-\eta^i)-\eta(1-\eta^{i-1})}{1-\eta^B}}(\log T)^{d\eta\frac{(\eta^i- \eta^B) - \eta(\eta^{i-1}- \eta^B)}{1-\eta^B}} (\log T)^{d\eta^2} \\
        & = 2 T^{\frac{1-\eta}{1-\eta^B}}(\log T)^{\frac{d\eta(\eta-\eta^B)}{1-\eta^B}},
    \end{align}    
    and the lower bound is proved similarly.
    
    \textbf{Proof of the fourth property.} This is analogous to the proof of the third property, and is thus omitted.
\end{proof}

In Section \ref{sec:standard_lower_bound}, we defined the $i$-th batch $\{t: t_{i-1} < t \leq t_i \}$ to be bad if $t_{i-1}<T_{i-1} $ and $t_i\geq T_i$. 
To make this more concrete, we define ``bad events'' $\{A_i\}_{i=1}^B $, where
\begin{align}
    A_1 & =\{t_1\geq T_1 \}, \\
    A_i & =\{t_{i-1}<T_{i-1},t_i\geq T_{i} \},
\end{align}
for $i\geq 2$. An important property of the bad events is that they partition the probability space, 
as $\cup_{i\in[B]} A_i = \{T_B \geq T\}$, which always holds since $T_B = T$. This means it is guaranteed that at least one bad event will happen, i.e., there will be a bad batch (in retrospect). We deduce that $\sum_{i=1}^B \mathbb{P}_0(A_i) \geq \mathbb{P}_0(\cup_{i\in[B]}A_i)=1$.

From the above definitions we derive the following lemmas.

\begin{lemma}
    \label{lem:aux_bad_event}
    Under the preceding definitions, for all $i\in[B]$, if $t \geq T_{i-1}$, then $\mathbb{P}_0^t(A_i) =  \mathbb{P}_0(A_i)$.
\end{lemma}
\begin{proof}
    Recall that $A_1=\{t_1\geq T_1 \} $ and $A_i=\{t_{i-1}<T_{i-1},t_i\geq T_{i} \}$ for $i\geq 2$.
    The claimed result follows directly from the fact that $A_i$ is determined at the start of time $T_{i-1} $: (i) If $t_{i-1}<T_{i-1}$, 
    then the algorithm has determined $t_i$ at the end of time $t_{i-1}$, so whether $t_i \geq T_i$ is known at the start of time $T_{i-1}$.
    (ii) If $t_{i-1}\geq T_{i-1}$, then it is known at the start of time $T_{i-1}$ that $A_i$ will not happen.  
\end{proof}

\begin{lemma}
    \label{lem:aux_simple_regret}
    For any $f_{m}\in\mathcal{F}_{i}$, we have $\mathbb{E}_m^t[r_t] \geq\epsilon_i \cdot \mathbb{P}_m^t(x_t \notin \mathcal{R}_m)$.
\end{lemma}
\begin{proof}
    Notice that the regret is non-negative, and is at least $\epsilon_i$ if $x\notin \mathcal{R}_m$ since all points outside $\mathcal{R}_m$ are at least $\epsilon_i$-suboptimal by Proposition \ref{prop:property_f}.
\end{proof}

\begin{lemma}
    \label{lem:aux_avg_tv_bound}
    Under the preceding definitions,
    for $\mathcal{F}_{i} = \{f_1,\ldots,f_{M_i}\}$ and any $t\in[T]$, we have
    \begin{align}
        \frac{1}{M_i}\sum_{m=1}^{M_i}\mathrm{TV}(\mathbb{P}_0^{t}, \mathbb{P}_m^{t})
        \leq \frac{C\epsilon_i}{2\sigma} \sqrt{\frac{t}{M_i}}
    \end{align}
    where $\mathrm{TV}$ is the total variation distance,
    and $C$ is a constant.
\end{lemma}

\begin{proof}
    We start with a single TV term and use Pinsker's inequality to get
    \begin{align}
        \mathrm{TV}(\mathbb{P}_0^{t}, \mathbb{P}_m^{t}) \leq \sqrt{\frac{1}{2}D_{\mathrm{KL}}(\mathbb{P}_0^{t}\Vert \mathbb{P}_m^{t})}, \label{eq:tv_to_kl}
    \end{align}
    where $D_{\mathrm{KL}}$ is the KL divergence. 
    We claim that this KL divergence can be upper bounded as follows:
    \begin{align}
        D_{\mathrm{KL}}(\mathbb{P}_0^{t}\Vert \mathbb{P}_m^{t})
        = \mathbb{E}_0^t\Bigl[ \log\Bigl(\frac{d\mathbb{P}^t_0}{d\mathbb{P}^t_m}\Bigr) \Bigr]
        \leq \sum_{s=1}^{t-1} \mathbb{E}_0^s\bigl[ D_{\mathrm{KL}}\left(P_0(\cdot|x_s), P_m(\cdot|x_s)\right) \bigr] \label{eq:decomposed_kl},
    \end{align}
    where the expectation is taken over the randomness of $x_s$.  Indeed, if \emph{all $t-1$ past observations} were available at the start of time $t$, then this inequality would hold \emph{with equality} by the chain rule for KL divergence (\cite{Cover06}, Theorem 2.5.3).  In our setting, under both $\mathbb{P}_0^t$ and $\mathbb{P}_m^t$, only observations from \emph{earlier} batches are available, and thus we get \eqref{eq:decomposed_kl} as an upper bound via the data processing inequality for KL divergence (\cite{Polyanskiy25}, Theorem 2.17) and the fact that the observed data at the start of time $t$ is a function of the entire set of data (including unobserved) at the start of time $t$. To simplify notation, we weaken the right-hand side of \eqref{eq:decomposed_kl} to $\sum_{s=1}^{t} \mathbb{E}_0^s\bigl[ D_{\mathrm{KL}}\left(P_0(\cdot|x_s), P_m(\cdot|x_s)\right) \bigr]$, i.e., we add the $t$-th term despite the fact that it is not strictly necessary.
    
    Next, define $\bar{D}^j_m := \max_{x \in \mathcal{R}_j} D_{\mathrm{KL}}( P_0(\cdot|x) \| P_m(\cdot|x) )$ to be the maximum KL divergence within $\mathcal{R}_j$ (defined in Proposition \ref{prop:property_f}), $\bar{v}^j_m := \max_{x \in \mathcal{R}_j} f_m(x) $ to be the maximum value of $f_{m}$ in $\mathcal{R}_j$, and $N_j^t = \sum_{s=1}^t \mathbb{I}(x_s\in\mathcal{R}_j) $ to be the number of points selected from $\mathcal{R}_j$ in the first $t$ rounds.
    We proceed by decomposing according to the region $x_t$ lies in:
    \begin{align}
        \sum_{s=1}^{t} \mathbb{E}_0^s\left[ D_{\mathrm{KL}}\left(P_0(\cdot|x_s), P_m(\cdot|x_s)\right) \right]
        & = \sum_{s=1}^{t} \mathbb{E}_0^s\biggl[ \sum_{j=1}^{M_i}\mathbb{I}(x_s\in\mathcal{R}_j) D_{\mathrm{KL}}\left(P_0(\cdot|x_s), P_m(\cdot|x_s)\right) \biggr] \\
        & = \sum_{j=1}^{M_i} \mathbb{E}_0^{t}\biggl[ \sum_{s=1}^{t} \mathbb{I}(x_s\in\mathcal{R}_j) D_{\mathrm{KL}}\left(P_0(\cdot|x_s), P_m(\cdot|x_s)\right) \biggr] \\
        & \leq \sum_{j=1}^{M_i} \mathbb{E}_0^{t}\biggl[ \sum_{s=1}^{t} \mathbb{I}(x_s\in\mathcal{R}_j) \bar{D}^j_m \biggr] \\
        & = \sum_{j=1}^{M_i} \mathbb{E}_0^{t}\bigl[ N^{t}_j \bigr] \bar{D}^j_m \\
        & = \frac{1}{2\sigma^2}\sum_{j=1}^{M_i} \mathbb{E}_0^{t}\bigl[ N^{t}_j \bigr] (\bar{v}^j_m)^2, \label{eq:use_v}
    \end{align}
    where \eqref{eq:use_v} follows since $\bar{D}^j_m = \max_{x \in \mathcal{R}_j} D_{\mathrm{KL}}( P_0(\cdot|x) \| P_m(\cdot|x) )= \max_{x \in \mathcal{R}_j} \frac{(f_m(x))^2}{2\sigma^2} = \frac{(\bar{v}^j_m)^2}{2\sigma^2} $, where the second equality follows from $f_0 = 0$ and (36) in \citep{Scarlett17}.
    Combining \eqref{eq:decomposed_kl} and \eqref{eq:use_v}, we have $D_{\mathrm{KL}}(\mathbb{P}_0^{t}\Vert \mathbb{P}_m^{t}) \leq \frac{1}{2\sigma^2}\sum_{j=1}^{M_i} \mathbb{E}_0^{t}\bigl[ N^{t}_j \bigr] (\bar{v}^j_m)^2$, and substituting back into \eqref{eq:tv_to_kl} gives
    \begin{align}
        \mathrm{TV}(\mathbb{P}_0^{t}, \mathbb{P}_m^{t}) 
        \leq \sqrt{\frac{1}{2}D_{\mathrm{KL}}(\mathbb{P}_0^{t}\Vert \mathbb{P}_m^{t})}
        \leq \frac{1}{2\sigma}\sqrt{\sum_{j=1}^{M_i} \mathbb{E}_0^{t}\bigl[ N^{t}_j \bigr] (\bar{v}^j_m)^2}.
    \end{align}
    Then, the average TV distance is bounded as follows:
    \begin{align}
        \frac{1}{M_i}\sum_{m=1}^{M_i}\mathrm{TV}(\mathbb{P}_0^{t}, \mathbb{P}_m^{t})
        & \leq \frac{1}{2\sigma}\cdot\frac{1}{M_i}\sum_{m=1}^{M_i}\sqrt{\sum_{j=1}^{M_i} \mathbb{E}_0^{t} \bigl[ N^{t}_j \bigr] (\bar{v}^j_m)^2} \\
        & \leq \frac{1}{2\sigma} \sqrt{\frac{1}{M_i}\sum_{m=1}^{M_i} \sum_{j=1}^{M_i} \mathbb{E}_0^{t} \bigl[ N^{t}_j \bigr] (\bar{v}^j_m)^2} \label{eq:jensen} \\
        & \leq \frac{1}{2\sigma} \sqrt{\frac{1}{M_i} \sum_{j=1}^{M_i} \mathbb{E}_0^{t} \bigl[ N^{t}_j \bigr] \sum_{m=1}^{M_i} (\bar{v}^j_m)^2} \\
        & \leq \frac{C\epsilon_i}{2\sigma} \sqrt{\frac{1}{M_i} \sum_{j=1}^{M_i} \mathbb{E}_0^{t} \bigl[ N^{t}_j \bigr]} \label{eq:use_lemma_5} \\
        & \leq \frac{C\epsilon_i}{2\sigma} \sqrt{\frac{t}{M_i}} \label{eq:add_nj},
    \end{align}
    where \eqref{eq:jensen} follows from Jensen's inequality, \eqref{eq:use_lemma_5} uses Lemma 5 in \citep{Scarlett17} stating that $\sum_{m=1}^{M_i} (\bar{v}^j_m)^2 \leq C \epsilon^2$ for all $j$ and some universal constant $C$, and \eqref{eq:add_nj} follows from $\sum_{j=1}^{M_i} N^{t}_j = t$.
\end{proof}

\begin{lemma}
    \label{lem:change_measure}
    Under the preceding definitions, if $i\in [B]$ and $t \in (T_{i-1}, T_i]$, then
    \begin{align}
        \mathbb{P}_m^t\bigl((x_t \notin \mathcal{R}_m )\cap A_i\bigr)
        \geq \mathbb{P}_0(A_i) - \frac{3}{2} \mathrm{TV}(\mathbb{P}_m^{T_{i-1}}, \mathbb{P}_0^{T_{i-1}}) - \mathbb{P}_0^t \bigl((x_t \in \mathcal{R}_m )\cap A_i\bigr)
    \end{align}
\end{lemma}
\begin{proof}
    Notice that
    \begin{align}
        & \mathbb{P}_m^t\bigl((x_t \notin \mathcal{R}_m )\cap A_i\bigr) + \mathbb{P}_0^t\bigl((x_t \in \mathcal{R}_m )\cap A_i\bigr) \nonumber \\
        & \qquad = \int_{(x_t \notin \mathcal{R}_m )\cap A_i} d\mathbb{P}_m^t + \int_{(x_t \in \mathcal{R}_m )\cap A_i} d\mathbb{P}_0^t \\
        & \qquad\geq  \int_{(x_t \notin \mathcal{R}_m )\cap A_i} \min\{d\mathbb{P}_m^t, d\mathbb{P}_0^t\} + \int_{(x_t \in \mathcal{R}_m )\cap A_i} \min\{d\mathbb{P}_m^t, d\mathbb{P}_0^t\} \\
        & \qquad = \int_{A_i} \min\{d\mathbb{P}_m^t, d\mathbb{P}_0^t\},
    \end{align}
    which implies
    \begin{align}
        \mathbb{P}_m^t\bigl((x_t \notin \mathcal{R}_m )\cap A_i\bigr)
        \geq \int_{A_i} \min\{d\mathbb{P}_m^t, d\mathbb{P}_0^t\} - \mathbb{P}_0^t \bigl((x_t \in \mathcal{R}_m )\cap A_i\bigr). \label{eq:min_dp}
    \end{align}
    To lower bound the integral in \eqref{eq:min_dp},
    notice that if $t \in (T_{i-1}, T_i]$ and $A_i = \{t_{i-1} < T_{i-1}, t_i \geq T_i \}$ happens, 
    then the algorithm has access to the same set of observations $y_1,\ldots, y_{t_{i-1}} $ at the start of time $t$ and at the start of time $T_{i-1}$ (or observes nothing if $i=1$),
    which implies the following (similar to step (c) of Eq.~(17) in \citep{Gao19}):
    \begin{align}
        \int_{A_i} \min\{d\mathbb{P}_m^t, d\mathbb{P}_0^t\}
        & = \int_{A_i} \min\{d\mathbb{P}_m^{T_{i-1}}, d\mathbb{P}_0^{T_{i-1}}\}. \label{eq:no_observation}
    \end{align}
    To proceed, we follow (18) in \citep{Gao19}:
    \begin{align}
        \int_{A_i} \min\{d\mathbb{P}_m^{T_{i-1}}, d\mathbb{P}_0^{T_{i-1}}\}
        & = \int_{A_i} \frac{1}{2} (d\mathbb{P}_m^{T_{i-1}} + d\mathbb{P}_0^{T_{i-1}} - \lvert d\mathbb{P}_m^{T_{i-1}} - d\mathbb{P}_0^{T_{i-1}} \rvert) \\
        & = \frac{1}{2} (\mathbb{P}_m^{T_{i-1}}(A_i) + \mathbb{P}_0^{T_{i-1}}(A_i)) - \frac{1}{2}\int_{A_i} \lvert d\mathbb{P}_m^{T_{i-1}} - d\mathbb{P}_0^{T_{i-1}} \rvert \\
        & \geq \mathbb{P}_0^{T_{i-1}}(A_i) - \frac{1}{2} (\mathbb{P}_0^{T_{i-1}}(A_i) - \mathbb{P}_m^{T_{i-1}}(A_i)) - \mathrm{TV}(\mathbb{P}_m^{T_{i-1}}, \mathbb{P}_0^{T_{i-1}}) \label{eq:tv_1} \\
        & \geq \mathbb{P}_0^{T_{i-1}}(A_i) - \frac{3}{2} \mathrm{TV}(\mathbb{P}_m^{T_{i-1}}, \mathbb{P}_0^{T_{i-1}}) \label{eq:tv_2} \\
        & = \mathbb{P}_0(A_i) - \frac{3}{2} \mathrm{TV}(\mathbb{P}_m^{T_{i-1}}, \mathbb{P}_0^{T_{i-1}}) \label{eq:min_dp_bound},
    \end{align}
    where \eqref{eq:tv_1} follows since $\mathrm{TV}(\mathbb{P}, \mathbb{Q}) = \frac{1}{2}\int\lvert d\mathbb{P} - d\mathbb{Q} \rvert \geq \frac{1}{2}\int_A \lvert d\mathbb{P} - d\mathbb{Q} \rvert $, \eqref{eq:tv_2} follows since $ \lvert \mathbb{P}(A)-\mathbb{Q}(A) \rvert \leq \mathrm{TV}\left(\mathbb{P}, \mathbb{Q}\right) $, and \eqref{eq:min_dp_bound} follows from Lemma \ref{lem:aux_bad_event}. Combining \eqref{eq:min_dp}, \eqref{eq:no_observation}, and \eqref{eq:min_dp_bound} leads to the desired result.
\end{proof}

\subsubsection{Completion of the Proof of Theorem \ref{thm:lower_adaptive_batches}}
\label{app:proof_thm_2}

We will give full proof for the SE kernel, then discuss the differences for the Mat\'ern kernel.
The idea is to show that for any algorithm that adaptively chooses batch sizes, its expected regret averaged over all functions in $\{\mathcal{F}_{i}\}_{i=1}^B$ is lower bounded. Our proof relies on the idea of ``change of measure'', which converts probabilities expressed in $\mathbb{P}_m$ (with $f=f_m\in \mathcal{F}_i$ for some batch index $i$) to $\mathbb{P}_0$ (with $f=0$).
For fixed $i\in[B]$, we lower bound the expected regret averaged over $f_{m} \in \mathcal{F}_{i}$ as follows:
\begin{align}
    \frac{1}{M_i}\sum_{m=1}^{M_i}\mathbb{E}_m[R_T]
    & = \frac{1}{M_i}\sum_{m=1}^{M_i}\sum_{t=1}^{T}\mathbb{E}_m^t[r_t] \\
    & \geq \frac{1}{M_i}\sum_{m=1}^{M_i}\sum_{t=T_{i-1}+1}^{T_i}\mathbb{E}_m^t[r_t] \\
    & \geq \epsilon_i\sum_{t=T_{i-1}+1}^{T_i} \frac{1}{M_i}\sum_{m=1}^{M_i} \mathbb{P}_m^t(x_t \notin \mathcal{R}_m) \label{eq:regret_to_prob} \\
    & \geq \epsilon_i\sum_{t=T_{i-1}+1}^{T_i} \frac{1}{M_i}\sum_{m=1}^{M_i} \mathbb{P}_m^t\bigl((x_t \notin \mathcal{R}_m) \cap A_i\bigr), \label{eq:to_bound}
\end{align}
where \eqref{eq:regret_to_prob} follows from Lemma \ref{lem:aux_simple_regret}.  We continue the above steps as follows:
\begin{align}
    \frac{1}{M_i}\sum_{m=1}^{M_i}\mathbb{E}_m[R_T]
    & \geq  \epsilon_i\sum_{t=T_{i-1}+1}^{T_i} \frac{1}{M_i}\sum_{m=1}^{M_i} \mathbb{P}_m^t\bigl((x_t \notin \mathcal{R}_m) \cap A_i\bigr) \\
    & \geq \epsilon_i\sum_{t=T_{i-1}+1}^{T_i} \frac{1}{M_i}\sum_{m=1}^{M_i} \Bigl(\mathbb{P}_0(A_i) - \frac{3}{2} \mathrm{TV}(\mathbb{P}_m^{T_{i-1}}, \mathbb{P}_0^{T_{i-1}}) - \mathbb{P}_0^t \bigl((x_t \in \mathcal{R}_m )\cap A_i\bigr)\Bigr) \label{eq:apply_TV_lem} \\
    & \geq \epsilon_i(T_i-T_{i-1}) \Bigl(\frac{M_i-1}{M_i}\mathbb{P}_0(A_i) - \frac{3}{2} \frac{1}{M_i}\sum_{m=1}^{M_i}\mathrm{TV}(\mathbb{P}_m^{T_{i-1}}, \mathbb{P}_0^{T_{i-1}})\Bigr) \label{eq:remove_t} \\
    & \geq \epsilon_i (T_i-T_{i-1}) \Bigl( \frac{M_i-1}{M_i} \mathbb{P}_0(A_i) - \frac{3C_1\epsilon_i}{4\sigma}\sqrt{\frac{T_{i-1}}{M_i}}\Bigr) \label{eq:bound_TV},
\end{align}
where \eqref{eq:apply_TV_lem} follows from Lemma \ref{lem:change_measure}, \eqref{eq:remove_t} follows from Lemma \ref{lem:aux_bad_event} along with $\sum_{m=1}^{M_i} \mathbb{P}_0^t \bigl((x_t \in \mathcal{R}_m )\cap A_i\bigr) = \mathbb{P}_0^t(A_i) = \mathbb{P}_0(A_i)$, and \eqref{eq:bound_TV} follows from Lemma \ref{lem:aux_avg_tv_bound} with $t = T_{i-1}$ and $C_1$ being the constant therein.\footnote{Here $C_{1}$, and the ensuing $C_{2}, C_{3}, \ldots$, are all constants, which we don't repeat to avoid clutter.} To proceed, we consider two cases for the batch index $i$.

\textbf{Case 1:} $i=1$. In this case $T_{i-1} = 0$, so
\begin{align}
    \frac{1}{M_1}\sum_{m=1}^{M_1}\mathbb{E}_m[R_T]
    \geq \epsilon_1 T_1 \frac{M_1-1}{M_1} \mathbb{P}_0(A_1)
    \geq \epsilon_1 \frac{M_1-1}{M_1} \cdot T^{\frac{1-\eta}{1-\eta^B}}(\log T)^{\frac{d\eta(\eta- \eta^B)}{1-\eta^B}} \mathbb{P}_0(A_1),
\end{align}
where the second inequality follows from the definition of $T_1$ in Proposition \ref{prop:reference_batches_property}. It remains to choose $\epsilon_1$, so we choose $\epsilon_1=\Theta(\Psi)=\Theta(1)$ (small enough to meet the requirement that $\frac{\epsilon_1}{\Psi}$ is sufficiently small). Also notice that $\frac{M_1-1}{M_1} \geq 3/4$ since $M_1 \gg 1$, which implies
\begin{align}
    \frac{1}{M_1}\sum_{m=1}^{M_1}\mathbb{E}_m[R_T]
    & \geq C_2\cdot T^{\frac{1-\eta}{1-\eta^B}}(\log T)^{\frac{d\eta(\eta- \eta^B)}{1-\eta^B}} \mathbb{P}_0(A_1). \label{eq:first_batch_regret}
\end{align}

\textbf{Case 2:} $i\geq 2$. 
Continuing from \eqref{eq:bound_TV}, we have
\begin{align}
    \frac{1}{M_i}\sum_{m=1}^{M_i}\mathbb{E}_m[R_T]
    & \geq \epsilon_i(T_i-T_{i-1})\Bigl(\frac{M_i-1}{M_i}\mathbb{P}_0(A_i)-\frac{3C_1\epsilon_i}{4\sigma}  \sqrt{\frac{T_{i-1}}{M_i}}\Bigr) \\
    & \geq \frac{3\epsilon_i}{4}(T_i-T_{i-1}) \Bigl(\mathbb{P}_0(A_i)- C_{3}\epsilon_i \sqrt{\frac{T_{i-1}}{M_i}}\Bigr). \label{eq:bound_to_tune}
\end{align}
where the second inequality follows since $\frac{M_i-1}{M_i} \geq \frac{3}{4}$ (since $M_i \gg 1$), and $\sigma$ is absorbed into $C_{3}$ since we assume it is constant.
It remains to choose $\epsilon_i$, subject to $\epsilon_i/\Psi$ being sufficiently small.
We choose $\epsilon_i$ to make the second term in the bracket in \eqref{eq:bound_to_tune} to equal $\frac{1}{2B}$, since this simplifies \eqref{eq:bound_to_tune} and aligns with the fact that $\mathbb{P}_0(A_i)$ is $1/B$ on average.
This leads to $\epsilon_i = \frac{\sqrt{M_i}}{2C_3 B\sqrt{T_{i-1}}}$. Then, substituting $M_i$ from \eqref{eq:m_se} leads to
\begin{align}
    \epsilon_i = C_{4} \cdot \sqrt{\frac{(\log(\Psi/\epsilon_i))^{d\eta}}{B^2T_{i-1}}},  \label{eq:epsilon_no_M}
\end{align}
which which is a "non-explicit" choice since it depends on $\log\frac{\Psi}{\epsilon_i}$. To remove this dependence, we study the log term as follows:
\begin{align}
    \log \frac{\Psi}{\epsilon_i}
    & = \log C_{5} + \frac{1}{2} \log \left( B^2\Psi^2T_{i-1} \right) - \frac{d}{4} \Bigl(\log \log\frac{\Psi}{\epsilon_i}\Bigr). \label{eq:remove_epsilon}
\end{align}
Since $\epsilon_i / \Psi$ is sufficiently small, it follows that $\log \frac{\Psi}{\epsilon_i}$ is sufficiently large. Also notice that $d$ is a constant, so we can consider $\log \log \frac{\Psi}{\epsilon_i} \leq \frac{2}{d}\log \frac{\Psi}{\epsilon_i}$. Substituting back into \eqref{eq:remove_epsilon} gives $\log\frac{\Psi}{\epsilon_i}= \Theta\bigl(\log (B^2\Psi^2T_{i-1}) \bigr)$, and thus \eqref{eq:epsilon_no_M} becomes
\begin{align}
    \epsilon_i 
    = C_{6} \cdot \sqrt{\frac{(\log B^2 \Psi^2 T_{i-1})^{d/2}}{B^2T_{i-1}}}.
\end{align}

Now, to ensure $\frac{\epsilon_i}{\Psi}$ is sufficiently small, we need $\Psi^{-1} =O(B\sqrt{T_{i-1}}) $ with a sufficiently small implied constant (since $t (\log (1/t))^{d/2} \to 0$ as $t \to 0$); this is trivially satisfied since we assume $\Psi$ is a constant. By absorbing $\Psi$ into the constant, we have
\begin{align}
    \epsilon_i 
    = C_{7} \cdot \sqrt{\frac{(\log B^2 T_{i-1})^{d/2}}{B^2T_{i-1}}}. \label{eq:epsilon_value}
\end{align}

Then, continuing from \eqref{eq:bound_to_tune} and recalling that we equated $C_{3}\epsilon_i \sqrt{\frac{T_{i-1}}{M_i}}$ with $\frac{1}{2B}$, we obtain
\begin{align}
    \frac{1}{M_i}\sum_{m=1}^{M_i}\mathbb{E}_m[R_T]
    & \geq \frac{3\epsilon_i}{4}(T_i-T_{i-1}) \Bigl(\mathbb{P}_0(A_i)- \frac{1}{2B} \Bigr) \label{eq:use_M} \\
    & = C_{8} B^{-1} (\log (B^2 T_{i-1}))^{d/4} (T_{i-1})^{-1/2}(T_i-T_{i-1}) \Bigl(\mathbb{P}_0(A_i)- \frac{1}{2B} \Bigr) \label{eq:use_epsilon} \\
    & \geq C_{9} B^{-1} (\log T)^{d/4} (T_{i-1})^{-1/2}(T_i-T_{i-1})\Bigl(\mathbb{P}_0(A_i)- \frac{1}{2B} \Bigr) \label{eq:remove_B} \\
    & = C_{10} B^{-1} T^{\frac{1-\eta}{1-\eta^B}}(\log T)^{\frac{d\eta(\eta-\eta^B)}{1-\eta^B}} \Bigl(\mathbb{P}_0(A_i)- \frac{1}{2B} \Bigr) \label{eq:remain_batch_regret},
\end{align}
where 
\eqref{eq:use_epsilon} substitutes \eqref{eq:epsilon_value}, 
\eqref{eq:remove_B} follows from $\log (B^2T_{i-1}) \geq \log T_{i-1} =\Theta(\log T)$,
and \eqref{eq:remain_batch_regret} follows from the first and third points of Proposition \ref{prop:reference_batches_property} (with $\eta = 1/2$).
Now, unifying $i=1$ in \eqref{eq:first_batch_regret} and $i\geq 2$ in \eqref{eq:remain_batch_regret} leads to
\begin{align}
    \frac{1}{M_i}\sum_{m=1}^{M_i}\mathbb{E}_m[R_T]
    & \geq C \cdot B^{-1} T^{\frac{1-\eta}{1-\eta^B}}(\log T)^{\frac{d\eta(\eta-\eta^B)}{1-\eta^B}} \Bigl(\mathbb{P}_0(A_i)- \frac{1}{2B} \Bigr)
\end{align}
for some constant $C$ and all $i\in[B]$. Averaging over $i\in [B]$ and applying $\sum_{i=1}^B \mathbb{P}_0(A_i) \geq 1 $ leads to
\begin{align}
    \frac{1}{B}\sum_{i=1}^B \frac{1}{M_i}\sum_{m=1}^{M_i}\mathbb{E}_m[R_T]
    & = \Omega\Bigl(B^{-2}T^{\frac{1-\eta}{1-\eta^B}}(\log T)^{\frac{d\eta(\eta-\eta^B)}{1-\eta^B}} \Bigr),
\end{align}
which implies the desired lower bound as the average regret lower bounds worst-case regret.

For the Mat\'ern kernel, the proof is entirely analogous but with fewer log terms, so we only highlight the differences:\footnote{We use $C$ (with subscripts) to refer to constants. If one such symbol is already used in the analysis for the SE kernel, its value here for the Mat\'ern kernel may be different.}

\textbf{Case 1:} $i=1$. By the same argument, we have
\begin{align}
    \frac{1}{M_1}\sum_{m=1}^{M_1}\mathbb{E}_m[R_T]
    \geq \epsilon_1 T_1 \mathbb{P}_0(A_1) \Bigl(1-\frac{1}{M_1}\Bigr)
    \geq C_1 \cdot T^{\frac{1-\eta}{1-\eta^B}}\mathbb{P}_0(A_1) \label{eq:first_batch_regret_mat}.
\end{align}

\textbf{Case 2:} $i\geq 2$. By the same argument, we arrive at \eqref{eq:bound_to_tune}:
\begin{align}
    \frac{1}{M_i}\sum_{m=1}^{M_i}\mathbb{E}_m[R_T]
    & \geq \frac{3\epsilon_i}{4}(T_i-T_{i-1}) \Bigl(\mathbb{P}_0(A_i)- C_2\epsilon_i  \sqrt{\frac{T_{i-1}}{M_i}}\Bigr). \label{eq:bound_to_tune_mat}
\end{align}
Then we choose $\epsilon_i$, which is simpler than the SE kernel case because there is no logarithmic dependence on $\epsilon_i$:
\begin{align}
    \epsilon_i 
    = \frac{\sqrt{M_i}}{2C_2B\sqrt{T_{i-1}}}
    = C_{3} \cdot \sqrt{\frac{\Psi^{\frac{1}{\eta}-2}(\epsilon_i)^{2-\frac{1}{\eta}}}{B^2T_{i-1}}}
    \Rightarrow
    \epsilon_i
    = C_{4} \cdot \Psi^{1-2\eta} (B^2T_{i-1})^{-\eta}.
\end{align}
Notice that $\Psi^{-1}=O(B\sqrt{T_{i-1}}) $ with a sufficiently small implied constant implies that $\epsilon_i/\Psi $ is sufficiently small, which is satisfied because $\Psi$ is assumed to be constant.
Then, absorbing $\Psi$ into the constant and substituting $\epsilon_i$ into \eqref{eq:bound_to_tune_mat} gives
\begin{align}
    \frac{1}{M_i}\sum_{m=1}^{M_i}\mathbb{E}_m[R_T]
    \geq C_{5} \cdot B^{-2\eta} T^{\frac{1-\eta}{1-\eta^B}} \Bigl(\mathbb{P}_0(A_i)-\frac{1}{2B}\Bigr) \label{eq:remain_batch_regret_mat}.
\end{align}
Unifying $i=1$ in \eqref{eq:first_batch_regret_mat} and $i\geq 2$ in \eqref{eq:remain_batch_regret_mat} leads to
\begin{align}
    \frac{1}{M_i}\sum_{m=1}^{M_i}\mathbb{E}_m[R_T]
    & \geq C \cdot B^{-2\eta} T^{\frac{1-\eta}{1-\eta^B}} \Bigl(\mathbb{P}_0(A_i)-\frac{1}{2B}\Bigr)
\end{align}
for some constant $C$ and all $i\in[B]$. Averaging over $i\in [B]$ and applying $\sum_{i=1}^B \mathbb{P}_0(A_i) \geq 1 $ leads to
\begin{align}
    \frac{1}{B}\sum_{i=1}^B \frac{1}{M_i}\sum_{m=1}^{M_i}\mathbb{E}_m[R_T]
    & = \Omega\Bigl(B^{-(2\eta+1)} T^{\frac{1-\eta}{1-\eta^B}} \Bigr),
\end{align}
which implies the desired result as average regret lower bounds worst-case regret.

\subsubsection{Completion of the Proof of Corollary \ref{cor:lower_adaptive_batches}}
\label{app:proof_cor}
Our proof is based on the arguments in \citep{Scarlett17} that derives high probability bounds from expected regret bounds in the fully sequential setting. We will first discuss this argument, then extend it to the setting with adaptive batches.

\textbf{High probability bounds in the fully sequential setting. }
In the fully sequential setting,  equation (79) in \citep{Scarlett17} shows that for any algorithm, there exists a $f\in\mathcal{F}$ such that $\mathbb{E}[R_T] \geq T\epsilon$ when $T = \Theta( M/\epsilon^2 )$. Recalling that $f\in[-2\epsilon, 2\epsilon]$, it follows that $R_T \leq 4T\epsilon$ with probability $1$. 
Now consider some $c\in (0,1)$. By applying the reverse Markov inequality (Markov's inequality applied to $4T\epsilon - R_T$), we have
\begin{align}
    \mathbb{P}(R_T \geq c T\epsilon)
    & = \mathbb{P}(4T\epsilon - R_T \leq 4T\epsilon -  c T\epsilon) \label{eq:rev_1} \\
    & = 1 - \mathbb{P}(4T\epsilon - R_T \geq 4T\epsilon -  c T\epsilon) \\
    & \geq 1 - \frac{4T\epsilon - \mathbb{E}[R_T]}{4T\epsilon -  c T\epsilon}
    = \frac{\mathbb{E}[R_T] - cT\epsilon}{4T\epsilon -  c T\epsilon} \\
    & \geq \frac{T\epsilon - cT\epsilon}{4T\epsilon -  c T\epsilon}
    = \frac{1 - c}{4 -  c}. \label{eq:rev_2}
\end{align}
Then, by choosing $c$ close to $0$, we have $R_T = \Omega (T\epsilon)$ with probability at least $\delta$, where $\delta$ can be any constant probability below $\frac{1}{4}$.

\textbf{Extension to adaptive batches.} We focus on the SE kernel, as the analysis for the Mat\'ern kernel is entirely analogous.
For the bad events $A_1, \ldots, A_B$, since $\sum_{i=1}^B \mathbb{P}_0(A_i) \geq 1 $, there exists $i\in [B]$ such that $\mathbb{P}_0(A_i) \geq 1/B$. 
To proceed, we consider two cases for $i$.

If $i\geq 2$, then from the proof in Appendix \ref{app:proof_thm_2}, we lower bound $R_T$ by regret in the $i$-th range of the reference batch sizes, i.e., $\{t:T_{i-1}<t\leq T_i\} $, which we denote as $\bar{R}^i$:
\begin{align}
    \frac{1}{M_i}\sum_{m=1}^{M_i}\mathbb{E}_m[R_T]
    \geq \frac{1}{M_i}\sum_{m=1}^{M_i}\sum_{t=T_{i-1}+1}^{T_i}\mathbb{E}_m^t[r_t]
    = \frac{1}{M_i}\sum_{m=1}^{M_i} \mathbb{E}_m[\bar{R}^i],
\end{align}
where in the rightmost expectation we note that  $\mathbb{E}^{T_i}_m[\bar{R}^i] = \mathbb{E}^T_m[\bar{R}^i] = \mathbb{E}_m[\bar{R}^i]$ since $\bar{R}^i$ only depends on points chosen up to time $T_i$. 
Then we lower bound $\bar{R}^i$ as follows:
\begin{align}
    \frac{1}{M_i}\sum_{m=1}^{M_i} \mathbb{E}_m[\bar{R}^i]
    & \geq \frac{3\epsilon_i}{4}(T_i-T_{i-1}) \Bigl(\mathbb{P}_0(A_i)- \frac{1}{2B} \Bigr) \label{eq:follow_proof} \\
    & \geq \frac{3\epsilon_i}{8B}(T_i-T_{i-1}) \label{eq:add_indicator} \\
    & = \Omega\Bigl(B^{-2}T^{\frac{1-\eta}{1-\eta^B}}(\log T)^{\frac{d\eta(\eta-\eta^B)}{1-\eta^B}} \Bigr), \label{eq:follow_proof_2}
\end{align}
where \eqref{eq:follow_proof} follows from \eqref{eq:use_M}, \eqref{eq:add_indicator} uses $\mathbb{P}_0^T(A_i) \geq 1/B$, and \eqref{eq:follow_proof_2} follows from \eqref{eq:remain_batch_regret}.
Notice that by the construction of $\mathcal{F}_i$, every $f\in \mathcal{F}_i$ satisfies $f\in[-2\epsilon_i, 2\epsilon_i]$, which means $\bar{R}^i \leq 4\epsilon_i(T_i-T_{i-1})$ with probability $1$.
Then, defining $L_i = \frac{3\epsilon_i}{8B}(T_i-T_{i-1})$, following the argument in \eqref{eq:rev_1}-\eqref{eq:rev_2} leads to
\begin{align}
    \frac{1}{M_i}\sum_{m=1}^{M_i}\mathbb{P}_m(\bar{R}^i\geq cL_i)
    & \geq \frac{(1-c)L_i}{4(T_i-T_{i-1})\epsilon_i-cL_i}
    = \frac{1-c}{(32/3)B-c}
    > \frac{1-c}{11B-c}.
\end{align}
This means by choosing $c$ close to $0$, there exists $f\in\mathcal{F}_i$ such that
\begin{align}
    R_T
    \geq \bar{R}^i
    \geq cL_i
    = \Omega\Bigl(B^{-2}T^{\frac{1-\eta}{1-\eta^B}}(\log T)^{\frac{d\eta(\eta-\eta^B)}{1-\eta^B}} \Bigr)
\end{align}
with probability at least $\delta$, where $\delta$ can be any probability below $\frac{1}{11B}$.

If $i=1$, then by following the analysis in Appendix \ref{app:proof_thm_2}, we have
\begin{align}
    \frac{1}{M_1}\sum_{m=1}^{M_1} \mathbb{E}_m[\bar{R}^1]
    \geq \epsilon_1 T_1 \frac{M_1-1}{M_1} \mathbb{P}_0(A_1)
    \geq \frac{3\epsilon_1}{8B} T_1
    = \Omega \Bigl(B^{-1} T^{\frac{1-\eta}{1-\eta^B}}(\log T)^{\frac{d\eta(\eta- \eta^B)}{1-\eta^B}}\Bigr).
\end{align}
Then, denote $L_1 = \frac{3\epsilon_1}{8B} T_1$, following the above argument for $i\geq 2$ leads to 
\begin{align}
    \frac{1}{M_1}\sum_{m=1}^{M_1}\mathbb{P}_m(\bar{R}^1\geq cL_1)
    & \geq \frac{(1-c)L_1}{4T_1\epsilon_1-cL_1}
    > \frac{1-c}{11B-c}.
\end{align}
This means by choosing $c$ close to $0$, there exists $f\in\mathcal{F}_i$ such that
\begin{align}
    R_T
    & \geq \bar{R}^1
    \geq cL_1
    = \Omega\Bigl(B^{-1}T^{\frac{1-\eta}{1-\eta^B}}(\log T)^{\frac{d\eta(\eta-\eta^B)}{1-\eta^B}} \Bigr)
    = \Omega\Bigl(B^{-2}T^{\frac{1-\eta}{1-\eta^B}}(\log T)^{\frac{d\eta(\eta-\eta^B)}{1-\eta^B}} \Bigr)
\end{align}
with probability at least $\delta$, where $\delta$ is any probability below $\frac{1}{11B}$.

\subsection{Proof of Theorem \ref{thm:upper_robust} (Robust Setting Upper Bounds)}
\label{app:proof_upper_robust}

The proof of Theorem \ref{thm:upper_robust} is based on several auxiliary results, which we present as follows.

\begin{proposition}
    \label{prop:rbpe_confidence_bounds}
    For any $\delta \in (0,1)$,
    running the robust-BPE algorithm with $\beta = \bigl( \Psi + \sqrt{2 \log \frac{|\mathcal{X}| B}{\delta}} \bigr)^2$, it holds with probability least $1-\delta$ that $f(x)\in[\mathrm{LCB}_i(x),\mathrm{UCB}_i(x)]$ for all $x\in\mathcal{X}$ and for all $i\in[B]$, i.e., we have valid confidence bounds in all batches.
\end{proposition}
\begin{proof}
    This result is essentially the same as
    Proposition \ref{prop:bpe_confidence_bounds}, because just like BPE, robust-BPE does maximum variance sampling in each batch, which allows the use of Lemma \ref{lem:aux_conf}.
\end{proof}

We henceforth assume valid confidence bounds in all batches for the robust-BPE algorithm.
Then, the following lemma shows that an $\xi$-optimal point can never be eliminated by robust-BPE.

\begin{lemma}
    \label{lem:aux_robust_elimination}
    After all $N_i$ points are sampled in the $i$-th batch, if $x\in\mathcal{X}_i$ satisfies
    \begin{align}
        \min_{\delta\in\Delta_{\xi}(x)} \mathrm{UCB}_i(x+\delta) 
        < \max_{x\in{\mathcal{X}_i}} \min_{\delta\in\Delta_{\xi}(x)} \mathrm{LCB}_i(x+\delta), \label{eq:rbpe_subopt}
    \end{align}
    then $x$ is not $\xi$-optimal.
\end{lemma}
\begin{proof}
    Denote $x' = \arg\max_{x\in{\mathcal{X}_i}} \min_{\delta\in\Delta_{\xi}(x)} \mathrm{LCB}_i(x+\delta) $. It suffices to show that when $x\in \mathcal{X}_i$ satisfies \eqref{eq:rbpe_subopt}, we have
    \begin{align}
        \min_{\delta\in\Delta_{\xi}(x)} f(x+\delta)
        < \min_{\delta\in\Delta_{\xi}(x')} f(x'+\delta).
    \end{align}
    We further define the following:
    \begin{align}
        \delta_f & = \argmin_{\delta\in\Delta_{\xi}(x)} f(x+\delta), \\
        \delta_U & = \argmin_{\delta\in\Delta_{\xi}(x)} \mathrm{UCB}_i(x+\delta), \\
        \delta'_L & = \argmin_{\delta\in\Delta_{\xi}(x')} \mathrm{LCB}_i(x'+\delta), \\
        \delta'_f & = \argmin_{\delta\in\Delta_{\xi}(x')} f(x'+\delta).
    \end{align}
    Then, by using the above definitions and the validity of the confidence bounds, we have
    \begin{align}
        \min_{\delta\in\Delta_{\xi}(x)} f(x+\delta) 
        & = f(x+\delta_f) \\
        & \leq f(x+\delta_U)
        \leq \mathrm{UCB}_i(x+\delta_U)
        = \min_{\delta\in\Delta_{\xi}(x)} \mathrm{UCB}_i(x+\delta) \\
        & < \min_{\delta\in\Delta_{\xi}(x')} \mathrm{LCB}_i(x'+\delta)
        = \mathrm{LCB}_i(x'+\delta'_L) \\
        & \leq \mathrm{LCB}_i(x'+\delta'_f)
        \leq f(x'+\delta'_L) = \min_{\delta\in\Delta_{\xi}(x')} f(x'+\delta).
    \end{align}
\end{proof}

The next lemma shows after maximum uncertainty sampling in a batch, all the remaining points after elimination have a uniform upper bound on the posterior variance.
\begin{lemma}
    \label{lem:aux_robust_posterior_variance}
    After all $N_i$ points are sampled in the $i$-th batch,
    for every element of $\mathcal{X}_{\xi,i}$, i.e., $x+\delta$ where $x\in \mathcal{X}_i$ and $\delta\in\Delta_\xi(x)$,
    its posterior variance is upper bounded by $\sigma^i(x+\delta) \leq \sqrt{\frac{C\gamma_{N_i}}{N_i}} $, where $C = \frac{2}{\log(1+\sigma^{-2})} $.
\end{lemma}
\begin{proof}
    This result directly follows from Appendix B.4 in \citep{Cai21b}, which we repeat for sake of completeness, and we reuse the definition of posterior variance from Appendix \ref{app:bpe}.
    We start with a known fact \citep{Srinivas10, Garnett23} that $\sum_{t=1}^{N_i} (\sigma^i_{t-1}(x_t))^2 \leq C \gamma_{N_i} $, where $C = \frac{2}{\log(1+\sigma^{-2})}$. 
    Then dividing both sides by $N_i$ leads to 
    \begin{align}
        \frac{1}{N_i}\sum_{t=1}^{N_i} (\sigma^i_{t-1}(x_t))^2 \leq \frac{C \gamma_{N_i}}{N_i}.
    \end{align}
    Since we are doing maximum variance sampling, we have $(\sigma^i_{N_i-1}(x_{N_i}))^2 \leq \frac{C \gamma_{N_i}}{N_i}$. Then, for every $x+\delta\in\mathcal{X}_{\xi, i}$,
    \begin{align}
        \sigma^i (x+\delta) \leq \sigma^i_{N_{i}-1}(x+\delta) \leq \sigma^i_{N_{i}-1}(x_{N_i}) \leq \sqrt{\frac{C \gamma_{N_i}}{N_i}},
    \end{align}
    where the first inequality follows because posterior variance cannot increase upon adding points, and the second follows because $x_{N_i}$ is chosen using maximum variance sampling.
\end{proof}

The next lemma combines the previous lemmas to upper bound the cumulative $\xi$-regret in a batch.
\begin{lemma}
    \label{lem:aux_xi_regret}
    After all $N_i$ points are sampled in the $i$-th batch,
    for any point not eliminated in $\mathcal{X}_i$, it holds that
    \begin{equation}
        r_{\xi}(x) \leq 6\sqrt{\frac{C\beta\gamma_{N_i}}{N_i}},
    \end{equation}
    where $C=\frac{2}{\log(1+\sigma^{-2})}$.
\end{lemma}
\begin{proof}
    The proof is again inspired by Appendix B.4 in \citep{Cai21b}, but adds considerations for robustness. By Lemma \ref{lem:aux_robust_elimination}, a $\xi$-robust optimal point $x^*$ cannot be eliminated. Consider any suboptimal point $x$ that is also not eliminated at the end of this batch.
    We then define the following, similar to those in Lemma \ref{lem:aux_robust_elimination}:
    \begin{align}
        \delta_U & = \argmin_{\delta\in\Delta_{\xi}(x)} \mathrm{UCB}_i(x+\delta), \\
        \delta_f & = \argmin_{\delta\in\Delta_{\xi}(x)} f(x+\delta), \\
        \delta^*_f & = \argmin_{\delta\in\Delta_{\xi}(x^*)} f(x^*+\delta), \\
        \delta^*_L & = \argmin_{\delta\in\Delta_{\xi}(x^*)} \mathrm{LCB}_i(x^*+\delta).
    \end{align}
    Then, we have
    \begin{align}
        \min_{\delta\in\Delta_{\xi}(x)} \mathrm{UCB}_i(x+\delta)
        & = \mathrm{UCB}_i(x+\delta_U)
        = \mathrm{LCB}_i(x+\delta_U) + 2\sqrt{\beta}\sigma^i(x+\delta_U) \\
        & \leq f(x+\delta_U) + 2\sqrt{\beta}\sigma^i(x+\delta_U)
        \leq f(x+\delta_U) + 2\sqrt{\frac{C\beta\gamma_{N_i}}{N_i}}, \label{eq:rbpe_1}
    \end{align}
    where \eqref{eq:rbpe_1} first uses the assumed confidence bound, and then uses Lemma \ref{lem:aux_robust_posterior_variance} and introduces $C = \frac{2}{\log(1+\sigma^{-2})} $.
    We  proceed to bound the term $f(x+\delta_U)$:
    \begin{align}
        f(x+\delta_U)
        & = f(x+\delta_f) + [f(x+\delta_U) - f(x+\delta_f)] \\
        & = f(x^*+\delta^*_f) - r_{\xi}(x) + [f(x+\delta_U) - f(x+\delta_f)] \label{eq:rbpe_use_def} \\
        & \leq f(x^*+\delta^*_f) - r_{\xi}(x) + [\mathrm{UCB}_i(x+\delta_U) - \mathrm{LCB}_i(x+\delta_f)] \label{eq:rbpe_2} \\
        & \leq f(x^*+\delta^*_f) - r_{\xi}(x) + [\mathrm{UCB}_i(x+\delta_f) - \mathrm{LCB}_i(x+\delta_f)] \label{eq:use_delta_u} \\
        & = f(x^*+\delta^*_f) - r_{\xi}(x) + 2\sqrt{\beta}\sigma^i(x+\delta_f) \\
        & \leq f(x^*+\delta^*_f) - r_{\xi}(x) + 2\sqrt{\frac{C\beta\gamma_{N_i}}{N_i}} \label{eq:rbpe_3},
    \end{align}
    where \eqref{eq:rbpe_use_def} uses the definition of $r_\xi(x)$,
    \eqref{eq:rbpe_2} uses the assumed confidence bound, 
    \eqref{eq:use_delta_u} uses the definition of $\delta_U$,
    and \eqref{eq:rbpe_3} uses Lemma \ref{lem:aux_robust_posterior_variance}.
    We further bound $f(x^*+\delta^*_f)$ as follows:
    \begin{align}
        f(x^*+\delta^*_f)
        & \leq f(x^*+\delta^*_L)
        \leq \mathrm{UCB}_i(x^*+\delta^*_L)
        = \mathrm{LCB}_i(x^*+\delta^*_L) + 2\sqrt{\beta}\sigma^i(x^*+\delta^*_L) \\
        & \leq \mathrm{LCB}_i(x^*+\delta^*_L) + 2\sqrt{\frac{C\beta\gamma_{N_i}}{N_i}} \label{eq:rbpe_4},
    \end{align}
    where \eqref{eq:rbpe_4} uses Lemma \ref{lem:aux_robust_posterior_variance}.
    Combining \eqref{eq:rbpe_1}, \eqref{eq:rbpe_3} and \eqref{eq:rbpe_4}, along with the definition of $\delta^*_L$, we have
    \begin{align}
        \min_{\delta\in\Delta_{\xi}(x)} \mathrm{UCB}_i(x+\delta)
        \leq \min_{\delta\in\Delta_{\xi}(x^*)} \mathrm{LCB}_i(x^*+\delta) - r_{\xi}(x) + 6\sqrt{\frac{C\beta\gamma_{N_i}}{N_i}}.
    \end{align}
    Since $x$ is not eliminated, we must have $- r_{\xi}(x) + 6\sqrt{\frac{C\beta\gamma_{N_i}}{N_i}} \geq 0$.
\end{proof}

We are now ready to prove Theorem \ref{thm:upper_robust}.
\begin{proof} [Proof of Theorem \ref{thm:upper_robust}]
    Recall that $\delta_f = \argmin_{\delta\in\Delta_{\xi}(x)} f(x+\delta)$ and $\delta^*_f = \argmin_{\delta\in\Delta_{\xi}(x^*)} f(x^*+\delta)$.
    In the first batch, we have
    \begin{align}
        r_{\xi}(x) 
        & = \min_{\delta\in\Delta_{\xi}(x^*)}f(x^*+\delta) - \min_{\delta\in\Delta_{\xi}(x)}f(x+\delta) \\
        & = f(x^*+\delta^*_f) - f(x+\delta_f) \\
        & \leq \mathrm{UCB}_1(x^*+\delta^*_f) - \mathrm{LCB}_1(x+\delta_f) \label{eq:use_conf} \\
        & \leq (\mu_0(x^*+\delta^*_f)+\sqrt{\beta}\sigma_0(x^*+\delta^*_f)) - (\mu_0(x+\delta_f)-\sqrt{\beta}\sigma_0(x+\delta_f)) \\
        & \leq 2\sqrt{\beta} \label{eq:prior} ,
    \end{align}
    where \eqref{eq:use_conf} uses the assumed confidence bound, and \eqref{eq:prior} follows because $\mu_0(\cdot) = 0$ (from Appendix \ref{app:bpe}) and $\sigma_0(\cdot) = \sqrt{k(\cdot,\cdot)} \leq 1$ (from Section \ref{sec:setup_standard}),
    which means the cumulative regret in the first batch is $R_{\xi}^1=O(N_1\sqrt{\beta})$.
    For a remaining batch where $i=2,\ldots, B$, by using Lemma \ref{lem:aux_xi_regret} (and introducing $C=\frac{2}{\log(1+\sigma^{-2})}$), we have
    \begin{align}
        R_{\xi}^i
        \leq N_i \cdot 6\sqrt{\frac{C\beta\gamma_{N_{i-1}}}{N_{i-1}}},
    \end{align}
    where the bound is of the same order as those in Proposition \ref{prop:property_bpe}. Then, we can follow the steps in the proof of Theorem \ref{thm:upper_grow_b} to get the desired regret bound on $R_{\xi, T}$.
\end{proof}

\section{A Refinement of Robust-BPE}
\label{app:simplify_robust_bpe}
In this section, we discuss a modification of the robust-BPE algorithm presented in Section \ref{sec:robust}, which potentially leads to improved statistical efficiency as we can shrink uncertainty faster.
Recall that in the $i$-th batch, robust-BPE does maximum variance sampling from the following set (line \ref{line:set_to_explore}, Algorithm \ref{algo:rbpe}):
\begin{align}
    \mathcal{X}_{\xi, i}\gets \cup_{x\in\mathcal{X}_i}\{x+\delta:\delta\in\Delta_{\xi}(x)\}.
\end{align}
Now we claim that we only need to explore the (potentially) smaller set:
\begin{align}
    \mathcal{X}_{\xi, i}\gets \cup_{x\in\mathcal{X}_i} \Bigl\{x+\delta:\delta\in\Delta_{\xi}(x), \mathrm{LCB}_{i-1}(x+\delta) \leq \min_{\delta'\in \Delta_{\xi}(x)}\mathrm{UCB}_{i-1}(x+\delta') \Bigr\},
\end{align}
where the difference is that we have an additional condition, which we design such that if some $x+\delta$ is eliminated from $\mathcal{X}_{\xi, i}$, then $\delta$ is not a minimizer of $f(x+\delta)$. 
To see this, note that if $\mathrm{LCB}_{i-1}(x+\delta) > \min_{\delta'\in \Delta_{\xi}(x)}\mathrm{UCB}_{i-1}(x+\delta')$, then 
\begin{align}
    f(x+\delta)
    \geq \mathrm{LCB}_{i-1}(x+\delta) 
    > \min_{\delta'\in \Delta_{\xi}(x)} \mathrm{UCB}_{i-1}(x+\delta')
    = \mathrm{UCB}_{i-1}(x+\delta_U)
    \geq f(x+\delta_U).
\end{align}
We proceed to show that this change does not affect Theorem \ref{thm:upper_robust} and its analysis.
First, Lemma \ref{lem:aux_robust_elimination} still holds because the each point with the minimizing $\delta$, i.e., $x+\delta_f\in \mathcal{X}_{\xi, i}$ for each $x\in \mathcal{X}_i$, is not eliminated:
\begin{align}
    \mathrm{LCB}_{i-1}(x+\delta_f)
    \leq f(x+\delta_f)
    \leq f(x+\delta_U)
    \leq \mathrm{UCB}_{i-1}(x+\delta_U)
    = \min_{\delta'\in \Delta_{\xi}(x)}\mathrm{UCB}_{i-1}(x+\delta').
\end{align}
Second, Lemma \ref{lem:aux_robust_posterior_variance} still holds, 
because robust-BPE still samples points with maximum uncertainty from $\mathcal{X}_{\xi,i}$, so the same proof holds, and the upper bound on posterior variance follows.
Third, validating Lemma \ref{lem:aux_xi_regret} amounts to showing whenever we use Lemma \ref{lem:aux_robust_posterior_variance}, the point under consideration is indeed a member of $\mathcal{X}_{\xi, i} $. In other words, we need to show $x+\delta_U$, $x+\delta_f$, and $x+\delta_L$ are members of $\mathcal{X}_{\xi, i}$ for all $x\in\mathcal{X}_i$.
This is shown as follows:
\begin{gather}
    \mathrm{LCB}_{i-1}(x+\delta_U) 
    \leq \mathrm{UCB}_{i-1}(x+\delta_U) 
    = \min_{\delta'\in \Delta_{\xi}(x)}\mathrm{UCB}_{i-1}(x+\delta'), \\
    \mathrm{LCB}_{i-1}(x+\delta_f)
    \leq f(x+\delta_f)
    \leq f(x+\delta_U)
    \leq \mathrm{UCB}_{i-1}(x+\delta_U) 
    = \min_{\delta'\in \Delta_{\xi}(x)}\mathrm{UCB}_{i-1}(x+\delta'), \\
    \mathrm{LCB}_{i-1}(x+\delta_L)
    \leq \mathrm{LCB}_{i-1}(x+\delta_U)
    \leq \mathrm{UCB}_{i-1}(x+\delta_U)
    = \min_{\delta'\in \Delta_{\xi}(x)}\mathrm{UCB}_{i-1}(x+\delta').
\end{gather}
As a result, the proof of Theorem \ref{thm:upper_robust} is unaffected, and we have the same theoretical guarantee.

\section{Upper Bounds for a Constant Number of Batches}
\label{app:fixed_B_upper_bounds}
This section complements Theorem \ref{thm:upper_grow_b} and Theorem \ref{thm:upper_robust} (which have $B=\Theta(\log\log T)$) by providing upper bounds in the case of constant $B$.

\subsection{Standard Setting}
\label{app:standard_fixed_B_upper_bounds}
Similar to Section \ref{sec:standard_upper_bound}, we start by presenting existing results, and then discuss our refinements.
\subsubsection{Existing Results}
\label{app:fixed_B_batch_sizes_li}
For fixed $B\geq 2$, \citep{Li22} considers the following batch sizes:
\begin{itemize}
    \item for the SE kernel, let $N_i =\Bigl\lceil T^{\frac{1-\eta^i}{1-\eta^B}}(\log T)^{\frac{d(\eta^i- \eta^B)}{1-\eta^B}} \Bigr\rceil$, where $\eta=\frac{1}{2}$;
    \item for the Mat\'ern kernel, let $N_i =\Bigl\lceil T^\frac{1-\eta^i}{1-\eta^B}\Bigr\rceil$, where $\eta = \frac{\nu}{2\nu+d}$;
\end{itemize}
for $i=1,\dots,B-1$, and $N_B=T-\sum_{j=1}^{B-1}N_j$. 
This construction allows the upper bounds of cumulative regret of the $i$-th batch to be of the same order, where $i=2,\ldots,B$. Then, \citep{Li22} presents the following upper bounds.

\begin{lemma}
    \emph{\citep[Theorem 2]{Li22}}
    \label{lem:upper_static_grid_li}
    Under the setup of Section \ref{sec:setup_standard} and using the batch sizes defined above,
    for any constant $B\geq 2$ and any $\delta \in (0,1)$,
    the BPE algorithm yields with probability at least $1-\delta$ that
    \begin{itemize}
        \item for the SE kernel,
        $R_T = O \Bigl( B(\Lambda+\sqrt{\log B}) T^{\frac{1-\eta}{1-\eta^B}} (\log T)^{\frac{d(\eta-\eta^B)}{1-\eta^B}+\eta}\Bigr)$, where $\eta=\frac{1}{2}$;
        \item for the Mat\'ern kernel, $R_T = O \Bigl(B (\Lambda+\sqrt{\log B}) T^{\frac{1-\eta}{1-\eta^B}} (\log T)^{\eta}\Bigr)$, where $\eta = \frac{\nu}{2\nu+d}$;
    \end{itemize}
    where $\Lambda=\Psi+\sqrt{\log(|\mathcal{X}|/\delta)}$.
\end{lemma}

\subsubsection{Refinements}
\label{app:fixed_B_refinements}
For the batch sizes defined above, we note that there is potentially an ``overflow'' issue that 
$N_B$ could be negative, 
in which case fewer than $B$ batches are actually used.\footnote{As an example, consider $T=1000$, $d=2$, $B=4$, and the SE kernel. Then, $N_B=-604$.}
We avoid this issue by first specifying the end of each batch, i.e., $t_i$ for $i\in[B]$, where $t_B=T$ and $t_0=0$, and then set $N_i = t_i-t_{i-1}$ (see Section \ref{sec:setup_standard}).
In addition, an inspection of the proof of Lemma \ref{lem:upper_static_grid_li} shows that the cumulative regret of the first batch is not of the same order as those of the remaining batches. By modifying the batch sizes to ensure every batch has the same order of regret, we obtain regret bounds with a slightly better dependence on $\log T$. 
Specifically, we consider the following batch sizes for an arbitrary constant $B\geq 2$:
\begin{itemize}
    \item for the SE kernel, let $t_i = \Bigl\lceil T^{\frac{1-\eta^i}{1-\eta^B}}(\log T)^{\frac{(d+1)(\eta^i- \eta^B)}{1-\eta^B}}\Bigr\rceil$, where $i\in[B]$ and $\eta=\frac{1}{2}$;
    \item for the Mat\'ern kernel, let $t_i = \Bigl\lceil T^{\frac{1-\eta^i}{1-\eta^B}}(\log T)^{\frac{\eta^i- \eta^B}{1-\eta^B}}\Bigr\rceil$, where $i\in[B]$ and $\eta=\frac{\nu}{2\nu+d}$.
\end{itemize}
Then we have the following result.
\begin{lemma}
    \label{lem:upper_static_grid}
    Under the setup of Section \ref{sec:setup_standard} and using batch sizes defined in Appendix \ref{app:fixed_B_refinements},
    for any constant $B\geq 2$ and any $\delta \in (0,1)$,
    the BPE algorithm yields with probability at least $1-\delta$ that
    \begin{itemize}
        \item for the SE kernel, $R_T = O \Bigl( B(\Lambda+\sqrt{\log B}) T^{\frac{1-\eta}{1-\eta^B}} (\log T)^{\frac{(d+1)(\eta-\eta^B)}{1-\eta^B}}\Bigr)$, where $\eta=\frac{1}{2}$;
        \item for the Mat\'ern kernel, $R_T = O \Bigl( B(\Lambda+\sqrt{\log B}) T^{\frac{1-\eta}{1-\eta^B}} (\log T)^{\frac{\eta-\eta^B}{1-\eta^B}}\Bigr)$, where $\eta = \frac{\nu}{2\nu+d}$;
    \end{itemize}
    where $\Lambda=\Psi+\sqrt{\log(|\mathcal{X}|/\delta)}$.
\end{lemma}
Notice that compared to Lemma \ref{lem:upper_static_grid_li}, Lemma \ref{lem:upper_static_grid} has regret bounds with a slightly better dependence on $\log T$. For example, we have $\frac{\eta-\eta^B}{1-\eta^B} \leq \frac{\eta-\eta^{B+1}}{1-\eta^B} = \eta $.

\subsubsection{Proof of Lemma \ref{lem:upper_static_grid}}
\label{app:proof_upper_fix_b}

We will give the proof for the SE kernel; the proof for the Mat\'ern kernel is entirely analogous and thus omitted. We will start with some properties of the batch sizes mentioned above.

\begin{proposition}
    \label{prop:property_static_grid}
    Let $t_i = \Bigl\lceil T^{\frac{1-\eta^i}{1-\eta^B}}(\log T)^{\frac{(d+1)(\eta^i- \eta^B)}{1-\eta^B}}\Bigr\rceil$, where $i\in[B]$ and $\eta=\frac{1}{2} $. Then,
    \begin{enumerate}
        \item $N_i=\Theta(t_i)$;
        \item $\log N_i=\Theta(\log T)$.
        \item For $2\leq i\leq B$, $t_i(t_{i-1})^{-\eta}(\log T)^{(d+1)\eta}=\Theta\Bigl(T^{\frac{1-\eta}{1-\eta^B}}(\log T)^{\frac{(d+1)(\eta-\eta^B)}{1-\eta^B}}\Bigr)$.
    \end{enumerate}
\end{proposition}
\begin{proof}
    The proof is entirely analogous to that of Proposition \ref{prop:reference_batches_property}.
\end{proof}

\textbf{Proof of Lemma \ref{lem:upper_static_grid}.}
Recall that we assume the kernel is SE.
For the first batch, we have
\begin{align}
    R^1 
    = O\bigl(N_1\sqrt{\beta}\bigr) 
    = O\Bigl((\Lambda+\sqrt{\log B})T^{\frac{1-\eta}{1-\eta^B}}(\log T)^{\frac{(d+1)(\eta- \eta^B)}{1-\eta^B}}\Bigr),
\end{align}
where $\eta=\frac{1}{2}$, the first equality follows from Proposition \ref{prop:property_bpe}, 
and the second follows from the definition of $\beta$ (in Appendix \ref{app:bpe}) and the definition of $N_1$ above (which is $t_1$).
For remaining batches where $2\leq i\leq  B$, we have
\begin{align}
    R^i 
    & = O \Bigl(\sqrt{\beta} \cdot N_i\sqrt{\gamma_{N_{i-1}}/N_{i-1}} \Bigr) \label{eq:fix_1} \\
    & = O \Bigl(N_i(N_{i-1})^{-\eta} (\log N_{i-1})^{(d+1)\eta}\sqrt{\beta}\Bigr) \\
    & = O \Bigl(t_i(t_{i-1})^{-\eta} (\log T)^{(d+1)\eta}\sqrt{\beta}\Bigr) \label{eq:fix_2} \\
    & = O\Bigl((\Lambda+\sqrt{\log B})T^{\frac{1-\eta}{1-\eta^B}} (\log T)^{\frac{(d+1)(\eta-\eta^B)}{1-\eta^B}} \Bigr) \label{eq:fix_3},
\end{align}
where \eqref{eq:fix_1} follows from Proposition \ref{prop:property_bpe}, \eqref{eq:fix_2} follows from the first two points of Proposition \ref{prop:property_static_grid}, and \eqref{eq:fix_3} follows from the third point of Proposition \ref{prop:property_static_grid} and the definition of $\beta$ (in Appendix \ref{app:bpe}).
Then we have the same order of the regret bound for every batch, which implies
\begin{align}
    R_T
    = \sum_{i=1}^B R^i
    = O\Bigl( B(\Lambda+\sqrt{\log B}) T^{\frac{1-\eta}{1-\eta^B}} (\log T)^{\frac{(d+1)(\eta-\eta^B)}{1-\eta^B}} \Bigr).
\end{align}

\subsection{Robust Setting}
\label{app:robust_fixed_B_upper_bounds}

The following result complements Theorem \ref{thm:upper_robust} by giving upper bounds in the robust setting for constant $B$, where the batch sizes are the same as those for Lemma \ref{lem:upper_static_grid}.

\begin{lemma}
    \label{lem:upper_robust_fixed_b}
    Under the setup of Section \ref{sec:setup_robust} and using batch sizes defined in Appendix \ref{app:fixed_B_refinements},
    for any constant $B\geq 2$ and any $\delta \in (0,1)$,
    the robust-BPE algorithm yields with probability at least $1-\delta$ that
    \begin{itemize}
        \item for the SE kernel, $R_{\xi,T} = O \Bigl( B(\Lambda+\sqrt{\log B}) T^{\frac{1-\eta}{1-\eta^B}} (\log T)^{\frac{(d+1)(\eta-\eta^B)}{1-\eta^B}}\Bigr)$, where $\eta=\frac{1}{2}$;
        \item for the Mat\'ern kernel, $R_{\xi,T} = O \Bigl( B(\Lambda+\sqrt{\log B}) T^{\frac{1-\eta}{1-\eta^B}} (\log T)^{\frac{\eta-\eta^B}{1-\eta^B}}\Bigr)$, where $\eta = \frac{\nu}{2\nu+d}$;
    \end{itemize}
    where $\Lambda=\Psi+\sqrt{\log(|\mathcal{X}|/\delta)}$.    
\end{lemma}

\begin{proof}
    The proof of Lemma \ref{lem:upper_robust_fixed_b} combines the proof of Theorem \ref{thm:upper_robust} in Appendix \ref{app:proof_upper_robust} and the proof of Lemma \ref{lem:upper_static_grid} in Appendix \ref{app:proof_upper_fix_b}. Focusing on the SE kernel, following the former leads to
    $R_{\xi}^1 = O (N_1\sqrt{\beta})$ and
    $R_{\xi}^i = O\bigl(\sqrt{\beta} N_i(N_{i-1})^{-\frac{1}{2}}(\log N_{i-1})^{\frac{d+1}{2}}\bigr)$ for $2\leq i \leq B$, and following the latter yields 
    \begin{align}
        R_{\xi}^i = O\Bigl((\Lambda+\sqrt{\log B})T^{\frac{1-\eta}{1-\eta^B}}(\log T)^{\frac{(d+1)(\eta- \eta^B)}{1-\eta^B}}\Bigr)
    \end{align}
    for $i\in [B]$, which implies
    \begin{align}
        R_{\xi,T} = O\Bigl(B(\Lambda+\sqrt{\log B})T^{\frac{1-\eta}{1-\eta^B}}(\log T)^{\frac{(d+1)(\eta- \eta^B)}{1-\eta^B}}\Bigr).
    \end{align}
\end{proof}

We observe that Lemma \ref{lem:upper_robust_fixed_b} has the same order of regret as Lemma \ref{lem:upper_static_grid}.

\section{Experiments}
\label{app:experiments}
While our work mainly focuses on improving the theoretical guarantees, we also perform simple proof-of-concept experiments to compare the parameterized batch sizes \eqref{eq:grow_b_batch_sizes} used in Theorem \ref{thm:upper_grow_b} with the original batch sizes \eqref{eq:grow_b_batch_sizes_li} in \citep{Li22}. While the experiments are not intended to be detailed or comprehensive, they illustrate the qualitative behavior predicted by our theory. 
Our code is available at \url{https://anonymous.4open.science/r/Batched-Kernelized-Bandits-0B0D}.

We consider three kernel settings (all using $l= 0.5$): the SE kernel, Mat\'ern kernel ($\nu=1.5$), and Mat\'ern kernel ($\nu=2.5$). For each kernel setting, we instantiate BPE with 4 types of batch sizes, including \eqref{eq:grow_b_batch_sizes_li} (referred to as ``Orig''), and \eqref{eq:grow_b_batch_sizes} with three different values of $a$ (specified in the following tables). These algorithms are compared on synthetic 2D functions sampled from GPs with the same kernel and $l=2.0$, as shown in Figure \ref{fig:syn_f}.\footnote{We note that the GP model used by BPE uses a shorter length scale $l=0.5$. This biases the GP model toward local variations, which helps to find the maximum more efficiently.} The domain $\mathcal{X}$ is a set of $2500$ points on a regular grid in $[-5,5]^2$. We let $T = 1000$ and $\sigma = 0.02$. Following \citep{Li22} (and other prior works), we set $\beta = 2$ instead of its theoretical value. In the tables below, the numbers are cumulative regret averaged over 10 trials at certain time steps. For convenience we include the number of batches $B$ in the first column.

\begin{figure}[!t]
    \begin{center}
        \begin{subfigure}[t]{0.28\textwidth}
            \begin{center}
                \centerline{\includegraphics[width=1.1\linewidth]{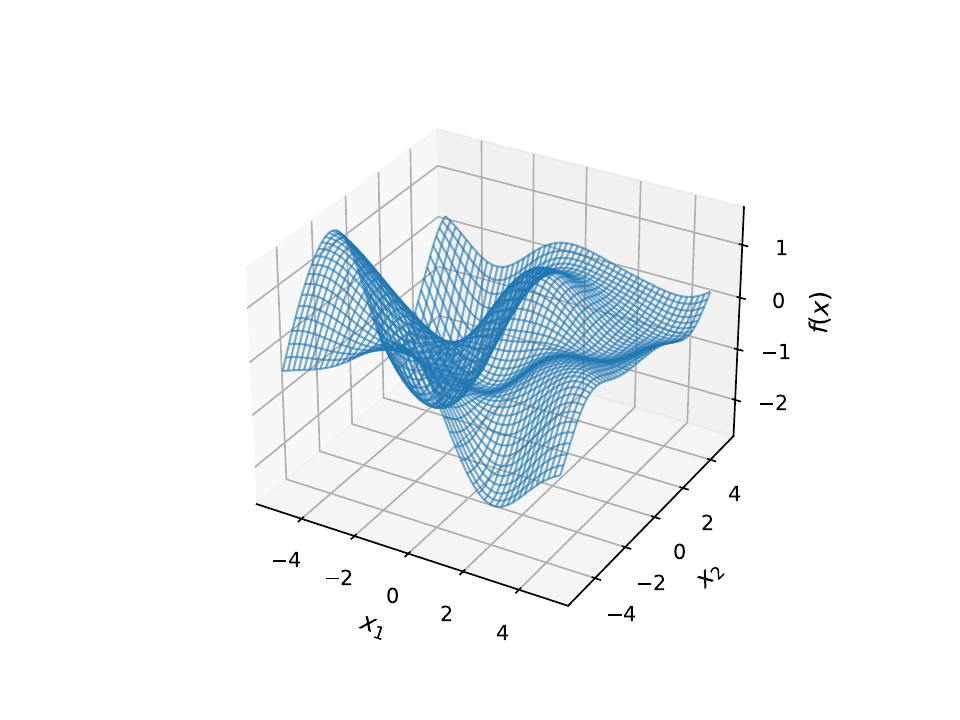}}
                \caption{SE kernel.}
                \label{fig:f_se}
            \end{center}
        \end{subfigure}
        \begin{subfigure}[t]{0.28\textwidth}
            \begin{center}
                \centerline{\includegraphics[width=1.1\linewidth]{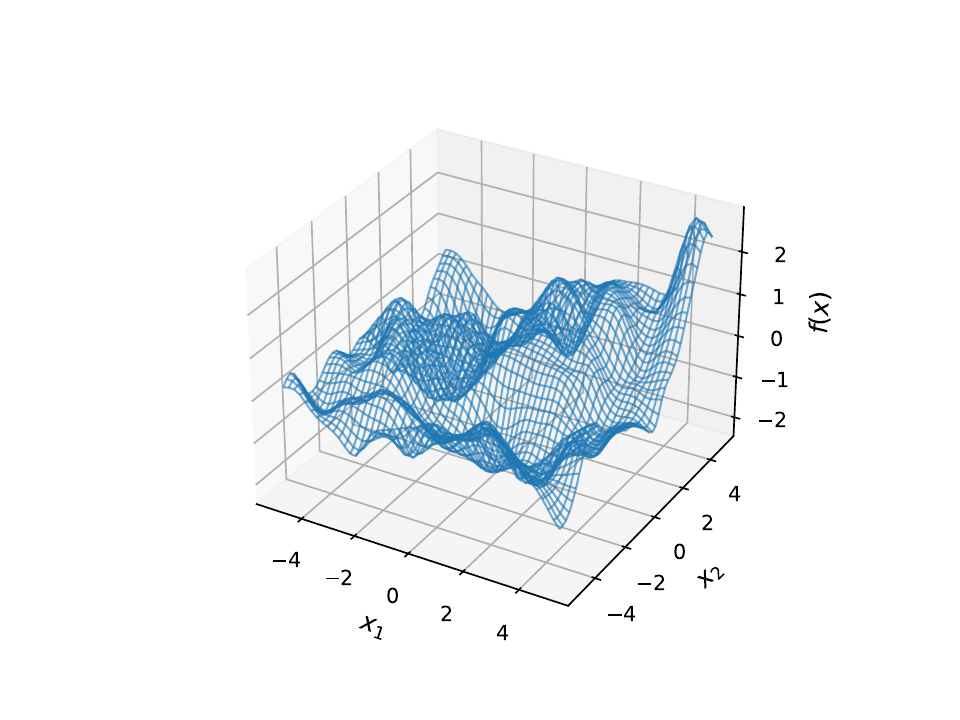}}
                \caption{Mat\'ern kernel ($\nu=1.5$).}
                \label{fig:f_mat_1.5}
            \end{center}
        \end{subfigure}
        \begin{subfigure}[t]{0.28\textwidth}
            \begin{center}
                \centerline{\includegraphics[width=1.1\linewidth]{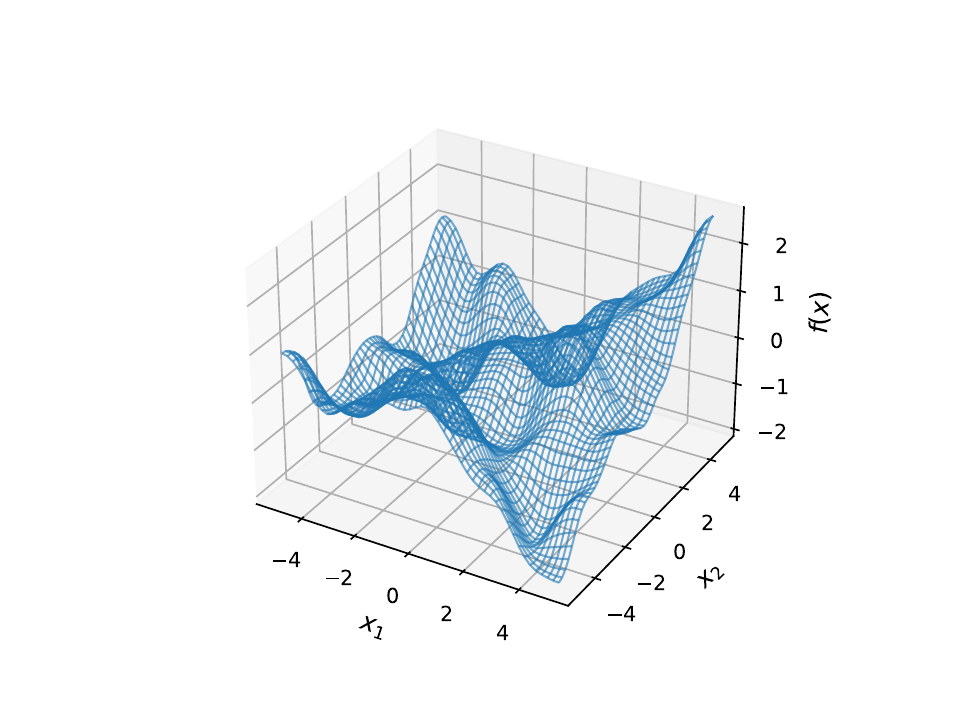}}
                \caption{Mat\'ern kernel ($\nu=2.5$).}
                \label{fig:f_mat_2.5}
            \end{center}
        \end{subfigure}
        \caption{Synthetic functions sampled using different kernels.}
        \label{fig:syn_f}
    \end{center}
\end{figure}

\begin{table}[ht]
\caption{Cumulative regret of BPE using different batch sizes and the SE kernel.} \label{tab:se}
\begin{center}
\begin{tabular}{llllll}
\toprule
Algorithm  & $T=200$ & $T=400$ & $T=600$ & $T=800$ & $T=1000$ \\
\midrule
BPE (Orig, $B=4$)       & 114.69 & 140.67 & 164.22 & 181.67 & 197.91 \\
BPE ($a=0.52$, $B=4$)   & 137.54 & 160.83 & 182.23 & 196.4  & 210.57 \\
BPE ($a=0.6$, $B=5$)    & 118.99 & 132.6  & 141.08 & 148.74 & 154.76 \\
BPE ($a=0.65$, $B=6$)   & 111.08 & 127.25 & 139.05 & 148.69 & 158.05 \\
\bottomrule
\end{tabular}
\end{center}
\end{table}

For the SE kernel, the $a$ with value $0.52$ is chosen close to its lower bound $0.5$ by Theorem \ref{thm:upper_grow_b}. As shown in Table \ref{tab:se}, this choice is comparable to “Orig”, while larger values improve regret at the expense of having more batches.

\begin{table}[h!]
\caption{Cumulative regret of BPE using different batch sizes and the Mat\'ern kernel ($\nu=1.5$).} \label{tab:mat_1.5}
\begin{center}
\begin{tabular}{llllll}
\toprule
Algorithm  & $T=200$ & $T=400$ & $T=600$ & $T=800$ & $T=1000$ \\
\midrule
BPE (Orig, $B=4$)     & 255.09 & 344.66 & 428.34 & 471.35 & 505.8  \\
BPE ($a=0.31$, $B=3$) & 367.95 & 471.04 & 574.1  & 630.37 & 677.65 \\
BPE ($a=0.4$, $B=3$)  & 254.45 & 344.24 & 384.13 & 424.07 & 464.1  \\
BPE ($a=0.5$, $B=4$)  & 254.31 & 334.58 & 409.87 & 443.63 & 469.56 \\
\bottomrule
\end{tabular}
\end{center}
\end{table}

For the Mat\'ern kernel ($\nu=1.5$), the two $a$ values of $0.31$ and $0.4$ are chosen closer to their lower bounds $0.3$ by Theorem \ref{thm:upper_grow_b}. As shown in Table \ref{tab:mat_1.5}, choosing $a=0.4$ leads to \emph{smaller regret compared to the original batch sizes, despite using fewer batches}, thus supporting our paper’s theoretical contribution of reducing the required $B$.

\begin{table}[h!]
\caption{Cumulative regret of BPE using different batch sizes and the Mat\'ern kernel ($\nu=2.5$).} \label{tab:mat_2.5}
\begin{center}
\begin{tabular}{llllll}
\toprule
Algorithm  & $T=200$ & $T=400$ & $T=600$ & $T=800$ & $T=1000$ \\
\midrule
BPE (Orig, $B=4$)     & 210.2  & 260.35 & 305.59 & 317.21 & 321.77 \\
BPE ($a=0.36$, $B=3$) & 226.77 & 246.05 & 256.91 & 260.48 & 264.06 \\
BPE ($a=0.4$, $B=3$)  & 192.91 & 218.55 & 220.44 & 222.34 & 224.23 \\
BPE ($a=0.5$, $B=3$)  & 222.34 & 261.57 & 294.47 & 304.41 & 310.01 \\
\bottomrule
\end{tabular}
\end{center}
\end{table}

For the Mat\'ern kernel ($\nu=2.5$), Table \ref{tab:mat_2.5} shows that choosing $a=0.4$ or $a=0.36$ leads to \emph{smaller regret compared to the original batch sizes, despite using fewer batches}, again supporting Theorem \ref{thm:upper_grow_b}.

\end{document}